\documentclass[11pt,letterpaper]{article}

\usepackage[utf8]{inputenc}
\usepackage{geometry}
\usepackage{amsmath,amssymb,amsfonts,amsthm}
\usepackage{graphicx}
\usepackage{booktabs}
\usepackage{tabularx}
\usepackage{natbib}
\usepackage{hyperref}
\usepackage{cleveref}
\usepackage{color}
\usepackage{titlesec}
\usepackage[font=footnotesize,labelfont=bf]{caption}
\newlength{\tablecaptiongap}
\AtBeginDocument{\begingroup\footnotesize\global\tablecaptiongap=.7\baselineskip\endgroup}
\newcommand{\papertablecaption}[1]{\setlength{\parskip}{0pt}\captionsetup{position=top,aboveskip=0pt,belowskip=0pt}\caption{#1}\par\nobreak\vskip\tablecaptiongap}
\newcommand{\tablesubheading}[1]{\par\noindent\strut\textsc{#1}\par\nobreak\vskip.7\baselineskip\nointerlineskip}

\title{\textbf{GeoCo-SAVi: Geometry-Consistent Slot Attention for\\ Explicitly Editable Object Representations}}
\author{\textbf{Haoxiang Huang}$^{1,\star}${\normalfont,}\enspace
  \textbf{Zhekai Wang}$^{2,\star}${\normalfont,}\enspace
  \textbf{Xiang Liu}$^{3}${\normalfont,}\enspace
  \textbf{Shuwei Wang}$^{4}${\normalfont,}\enspace
  \textbf{Yiwen Sun}$^{5}${\normalfont,}\\[0.47em]
  \textbf{Zhiwei Yu}$^{6}${\normalfont,}\enspace
  \textbf{Jingheng Ma}$^{3}${\normalfont,}\enspace
  \textbf{Sen Cui}$^{3,\dagger,\ddagger}${\normalfont,}\enspace
  \textbf{Changshui Zhang}$^{3,\dagger}$\\[0.76em]
  {\footnotesize\sffamily\mdseries\color{black}
    $^{1}$ University of Science and Technology of China{\normalfont,}\enspace
    $^{2}$ Beijing Institute of Technology{\normalfont,}\enspace
    $^{3}$ Tsinghua University{\normalfont,}}\\[0.30em]
  {\footnotesize\sffamily\mdseries\color{black}
    $^{4}$ Chinese Academy of Sciences{\normalfont,}\enspace
    $^{5}$ Peking University{\normalfont,}\enspace
    $^{6}$ Beijing Academy of Artificial Intelligence}}
\date{}

\usepackage{xcolor}
\usepackage{environ}
\usepackage{tcolorbox}
\usepackage{tikz}
\usetikzlibrary{positioning, arrows.meta, calc, fit, shapes.geometric, backgrounds}
\usepackage{titlesec}
\usepackage{colortbl}
\usepackage{array}
\usepackage{placeins}
\usepackage{needspace}
\usepackage{enumitem}

\definecolor{paperblue}{RGB}{41,105,176}
\definecolor{accent}{RGB}{230,90,50}
\definecolor{soft}{RGB}{232,240,250}
\definecolor{tableheadgray}{RGB}{246,246,246}
\definecolor{metricgain}{RGB}{0,145,90}
\definecolor{metricloss}{RGB}{205,45,55}
\colorlet{linkblue}{paperblue}

\hypersetup{
  colorlinks=true,
  linkcolor=linkblue,
  citecolor=linkblue,
  urlcolor=linkblue,
  filecolor=linkblue,
  linktoc=all,
  bookmarksopen=true,
  bookmarksnumbered=true,
  pdfborder={0 0 0}
}

\graphicspath{{figures/}}

\titleformat{\section}
  {\color{paperblue}\normalfont\Large\bfseries\sffamily}
  {\thesection}{1em}{}
\titleformat{\subsection}
  {\color{paperblue}\normalfont\large\bfseries\sffamily}
  {\thesubsection}{1em}{}
\titleformat{\subsubsection}
  {\color{paperblue}\normalfont\normalsize\bfseries\sffamily}
  {\thesubsubsection}{1em}{}
\titleformat{\paragraph}[runin]
  {\normalfont\normalsize\bfseries\sffamily}{}{0pt}{}
\titlespacing*{\section}{0pt}{2.8ex plus 0.6ex minus 0.2ex}{1.0ex plus 0.2ex}
\titlespacing*{\subsection}{0pt}{2.3ex plus 0.5ex minus 0.2ex}{0.8ex plus 0.1ex}
\titlespacing*{\subsubsection}{0pt}{1.9ex plus 0.4ex minus 0.2ex}{0.7ex plus 0.1ex}
\titlespacing{\paragraph}{0pt}{1.2ex plus 0.3ex minus 0.1ex}{0.7em}

\renewenvironment{abstract}{\begin{tcolorbox}[
    colframe=paperblue,
    colback=white,
    arc=8pt,
    boxrule=0.5pt,
    left=15pt,
    right=15pt,
    top=12pt,
    bottom=12pt,
    before upper={\small\setlength{\parskip}{4pt}\setlength{\parindent}{0pt}}
  ]
  \begin{center}
    {\color{paperblue}\large\bfseries\sffamily Abstract\par}
    \vskip 0.5em
  \end{center}
}{\end{tcolorbox}
}

\renewcommand{\textbf}[1]{{\sffamily\bfseries #1}}

\makeatletter
\renewcommand{\maketitle}{\newpage
  \thispagestyle{plain}\null
  \vskip -0.8em
  \setbox\@tempboxa=\vbox{\hsize=\textwidth
    \parskip=0pt
    \centering
    {\Large\sffamily\bfseries \@title\par}}\dimen@=\ht\@tempboxa
  \advance\dimen@ by \dp\@tempboxa
  \advance\dimen@ by 3em
  {\parskip=0pt
    \noindent\vbox{\hsize=\textwidth
      {\color{paperblue}\hrule height 2pt}\vbox to \dimen@{\vfil\box\@tempboxa\vfil}{\color{paperblue}\hrule height 0.5pt}}\par
  }\vskip 1.45em
  {\centering
    {\large\sffamily
      \lineskip .5em\begin{tabular}[t]{c}\@author
      \end{tabular}\par}}\vskip 1.55em
}
\makeatother

\newcommand{\sg}{\operatorname{sg}}

\newcommand{\logsumexp}{\operatorname{logsumexp}}
\newcommand{\Huber}{\mathrm H}
\newcommand{\R}{\mathbb{R}}

\newcommand{\cf}{\mathrm{cf}}
\newcommand{\rec}{\mathrm{rec}}
\newcommand{\don}{\mathrm{don}}
\newsavebox{\ablationleftbox}
\newsavebox{\ablationrightbox}
\newlength{\ablationheight}

\newtheoremstyle{paperbold}
  {6pt}{6pt}{\normalfont}{0pt}{\sffamily\bfseries}{.}{0.5em}{}
\theoremstyle{paperbold}

\newtheorem{proposition}{Proposition}
\newtheorem{lemma}{Lemma}
\newtheorem{theorem}{Theorem}
\newtheorem{remark}{Remark}

\begin{document}

\maketitle
\begin{abstract}
Object-centric video models represent scenes with slots, yet exposed geometry can vary in meaning with appearance.  In Invariant Slot Attention (ISA), explicit position and scale can disagree with the decoded center and extent; edits can yield unexpected motion or resizing, and replacing appearance can shift geometry.
\textbf{GeoCo-SAVi} promotes \textbf{geometric authority} and \textbf{semantic alignment}. Its spatially equivariant, object-wise decoder makes position and scale effective commands: changing them moves or resizes the rendered support.  Factual position alignment ties position to the decoded center, and normalized attention overlap discourages duplicate allocation. Appearance transplantation aligns geometry semantics across objects, so recipient geometry governs layout while donor appearance supplies shape.
On Obj3D, GeoCo-SAVi matches ISA reconstruction, reduces latent-position-to-decoded-centroid error by nearly 90\%, and reduces appearance-induced size variation while producing the expected translation and scale responses.  On 250 MOVi-C videos, it improves reconstruction and instance grouping over both same-protocol references, and video editing over STAITUS.  GeoCo-SAVi transforms explicit geometry into compositional control, making both position and scale more readable and editable.
\par\vspace{4pt}
\noindent\begin{minipage}[t]{0.60\linewidth}
\footnotesize
{\textbf{Date:}} September 26, 2026\\
{\textbf{Project:}} \href{https://GeoCo-SAVi.github.io}{\texttt{https://GeoCo-SAVi.github.io}}\\
{\textbf{Code:}} \href{https://github.com/starx237/GeoCo-SAVi}{\texttt{https://github.com/starx237/GeoCo-SAVi}}
\end{minipage}\hfill
\begin{minipage}[t]{0.37\linewidth}
\footnotesize\raggedleft
$^{\star}$ Equal Contribution\\
$^{\dagger}$ Corresponding Authors\\
$^{\ddagger}$ Project Leader
\end{minipage}
\end{abstract}

\tableofcontents
\clearpage

\section{Introduction}
\label{sec:introduction}

Slot-based models provide object-centric vectors for decomposition, tracking,
and compositional generation.  Their appeal for editing rests on a simple
promise: changing an object's representation should predictably change that
object.  Slot Attention~\citep{slotattention} discovers such vectors through
competitive binding, and SAVi~\citep{savi} propagates them through video.  Most
variants still mix appearance with image-plane geometry.

ISA~\citep{isa} improves this interface by exposing an appearance vector,
two-dimensional position, and scale for every slot.  Its relative
coordinates reduce appearance's geometric burden.  Editing the exposed
variables can still yield unexpected motion or resizing.  Their absolute
meanings may vary across appearances: position can miss the rendered center,
equal scales can yield different extents, and appearance swaps at fixed
geometry can move or resize objects
(Fig.~\ref{fig:teaser}).
\par\vspace{4pt}
\begingroup
    \centering
    \setlength{\abovecaptionskip}{13pt}
    \includegraphics[width=0.95\columnwidth]{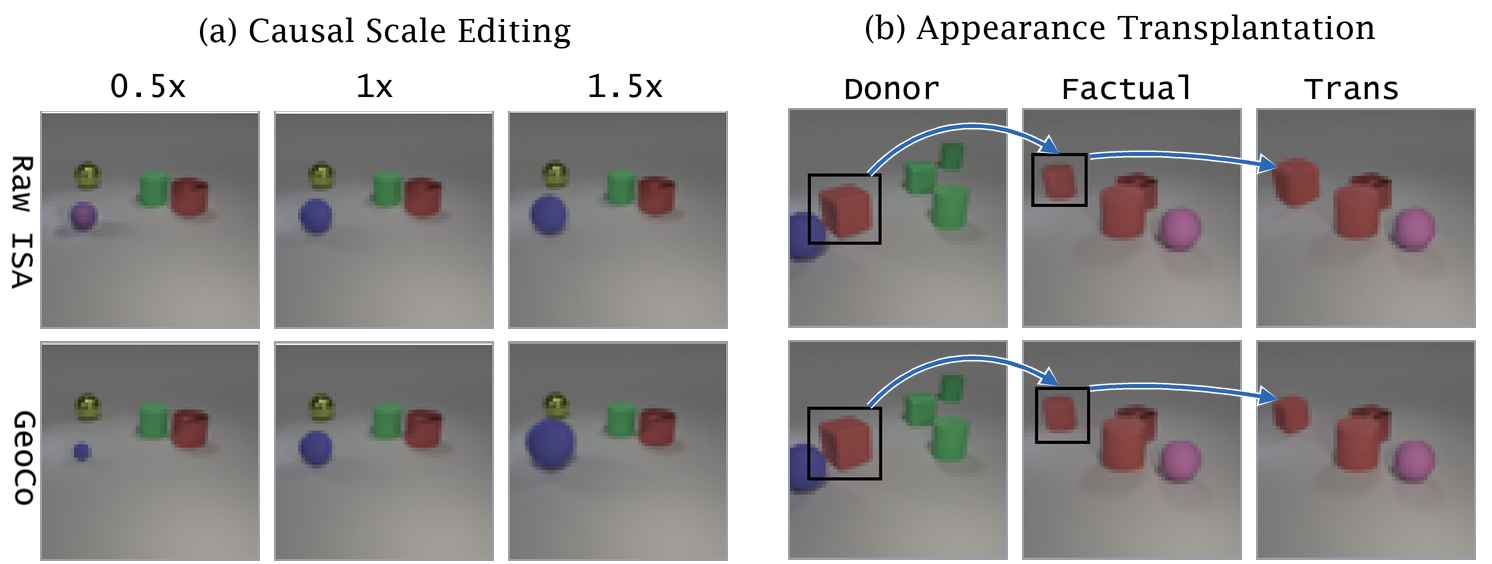}\par
    \captionof{figure}{Raw ISA (ISA in tables) responds weakly to scale commands and can shift geometry
    after transplantation.  GeoCo-SAVi produces predictable extent
    changes and preserves layout.}
    \label{fig:teaser}
\endgroup\par\vspace{4pt}

GeoCo-SAVi addresses this gap through a geometry-control formulation that combines
spatially equivariant, object-wise decoding with factual and counterfactual
calibration. The decoder makes each slot's position and scale effective spatial
commands: changing $\mathbf p_i$ translates its rendered support, while changing
$s_i$ resizes it around $\mathbf p_i$. Factual position alignment ties $\mathbf p_i$ to
the decoded center, and normalized attention overlap reduces duplicate
allocation.

Counterfactual supervision aligns geometry semantics across objects.
Appearance transplantation combines donor appearance with recipient
geometry, aligning decoded center and extent while retaining donor-dependent
support structure and compactness.  Appearance may remain statistically correlated with
geometry; we require predictable intervention semantics.

Our contributions are:
\begin{itemize}[leftmargin=1.2em,labelsep=0.55em,topsep=3pt,itemsep=2pt,parsep=0pt]
    \item We make ISA's position--scale representation geometrically calibrated and explicitly editable, with consistent semantics across appearances.
    \item We develop a geometry-control formulation combining spatially equivariant decoding with factual and counterfactual calibration, and characterize its ideal continuous raw-logit covariance.
    \item We demonstrate predictable position, scale, and appearance editing on Obj3D and MOVi-C, alongside competitive reconstruction and object grouping.
\end{itemize}

\section{Related Work}
\label{sec:related}

\paragraph{Object-centric image and video learning.}
Early generative models decompose scenes through sequential inference or mixture reconstruction~\citep{air,monet,genesis}. Slot Attention~\citep{slotattention} binds visual features to exchangeable slots, and SAVi~\citep{savi} extends this mechanism to conditional video tracking. Frozen self-supervised features improve grouping in realistic imagery~\citep{dinosaur}. Video methods build on self-supervision and temporal objectives~\citep{videosaur,solv,slotcontrast}, richer reasoning or synergistic learning~\citep{srl,ssync}, and reconstruction curricula~\citep{slotcurri}. STATM~\citep{statm} predicts subsequent slot states through spatiotemporal attention over a memory buffer, providing a modular reasoning component for object-centric video models. These advances target discovery and tracking; our focus is the intervention semantics of exposed geometry.

\paragraph{Explicit factors and controllable slots.}
AIR, SPACE, and SCALOR use what/where factors; ISA adds a slot-centric
appearance frame~\citep{air,space,scalor,isa}.
DISA and GOLD factor shape--texture and extrinsic--intrinsic attributes~\citep{disa,gold};
STAITUS and DSSA respectively separate appearance--pose and local appearance--identity~\citep{staitus,dssa}.
SlotAug uses augmentation commands and CTRL-O language selection; foreground
indicators refine support~\citep{slotaug,ctrlo,foregroundocl}. We test whether
$(\mathbf p,s)$ commands remain calibrated under novel appearance--geometry combinations.

\paragraph{Scale equivariance.}
Deep scale-space models, SESN, and DISCO are scale-equivariant
hierarchies~\citep{deepscalespaces,sesn,disco}.
SAC scales convolutional sampling within feature maps~\citep{sac}.
AIR and SPACE warp decoded object-local patches onto the image canvas
through position and scale transformations~\citep{air,space}.
Spatial broadcast decoding combines tiled latents with coordinates~\citep{spatialbroadcast}.

\section{Methodology}
\label{sec:method}

\subsection{Preliminary: Invariant Slot Attention}

For a frame $\mathbf x\in[0,1]^{3\times H\times W}$ on normalized grid
$\Omega$, an encoder produces tokens
$\{(\mathbf f_n,\mathbf u_n)\}_{n=1}^{L}$. Our isotropic ISA~\citep{isa} returns $N$
slot representations: $\mathbf a_i\in\R^{D_a}$ is appearance,
$\mathbf p_i\in[-1,1]^2$ is image-plane position, and $s_i>0$ is the scalar scale:
\begin{equation}
    \mathbf z_i=[\mathbf a_i;\mathbf p_i;s_i].
    \label{eq:slot-state}
\end{equation}
The shared decoder maps each slot to RGB $\mathbf y_i:\Omega\rightarrow[0,1]^3$ and a raw alpha logit $\ell_i:\Omega\rightarrow\R$. The core contains no learned foreground, background, or empty-slot gate. Full-softmax compositing gives
\begin{equation}
    \begin{aligned}
        \alpha_i(\mathbf u)
        &=\operatorname{softmax}(\ell(\mathbf u))_i
        ,\\
        \widehat{\mathbf{x}}(\mathbf u) &= \sum_i\alpha_i(\mathbf u)\mathbf y_i(\mathbf u).
    \end{aligned}
    \label{eq:compositor}
\end{equation}

\paragraph{Slot-centric inference.}

For slot $i$, ISA forms slot-centric token coordinates
\begin{equation}
    \boldsymbol\xi_i(\mathbf u_n)
    =\frac{\mathbf u_n-\mathbf p_i}{a_s s_i+\epsilon},
    \label{eq:isa-relative-coordinate}
\end{equation}
where $a_s=5$ fixes the coordinate scale. Its encoding enters attention keys and values. If $o_i(\mathbf u_n)$ is slot-softmax ownership, define
\begin{equation}
    w_i(\mathbf u_n)
    =\frac{o_i(\mathbf u_n)}{\sum_{m=1}^{L}o_i(\mathbf u_m)+\epsilon}
    \label{eq:spatial-attention}
\end{equation}
ISA updates appearance recurrently and reads geometry from the normalized moments
\begin{equation}
\begin{aligned}
    \mathbf p_i&=\sum_n w_i(\mathbf u_n)\mathbf u_n,
    &s_i&=\Big({\sum_n w_i(\mathbf u_n)
    \|\mathbf u_n-\mathbf p_i\|_2^2+\epsilon}\Big)^{1/2}.
\end{aligned}
\label{eq:isa-moments}
\end{equation}
Slot-indexed learnable appearance and scale initializations stabilize
reconstruction, while positions are initialized IID-randomly.  This relaxes
strict prior exchangeability; appearance transplantation remains valid.
Relative coordinates reduce appearance's geometric burden.  Decoder compliance
with $(\mathbf p_i,s_i)$ requires a commanded decoder.

Figure~\ref{fig:architecture} shows the core and calibration paths; see Sec.~\ref{sec:implementation-details} for implementation details.

\begin{figure}[t]
    \centering
    \includegraphics[width=\textwidth]{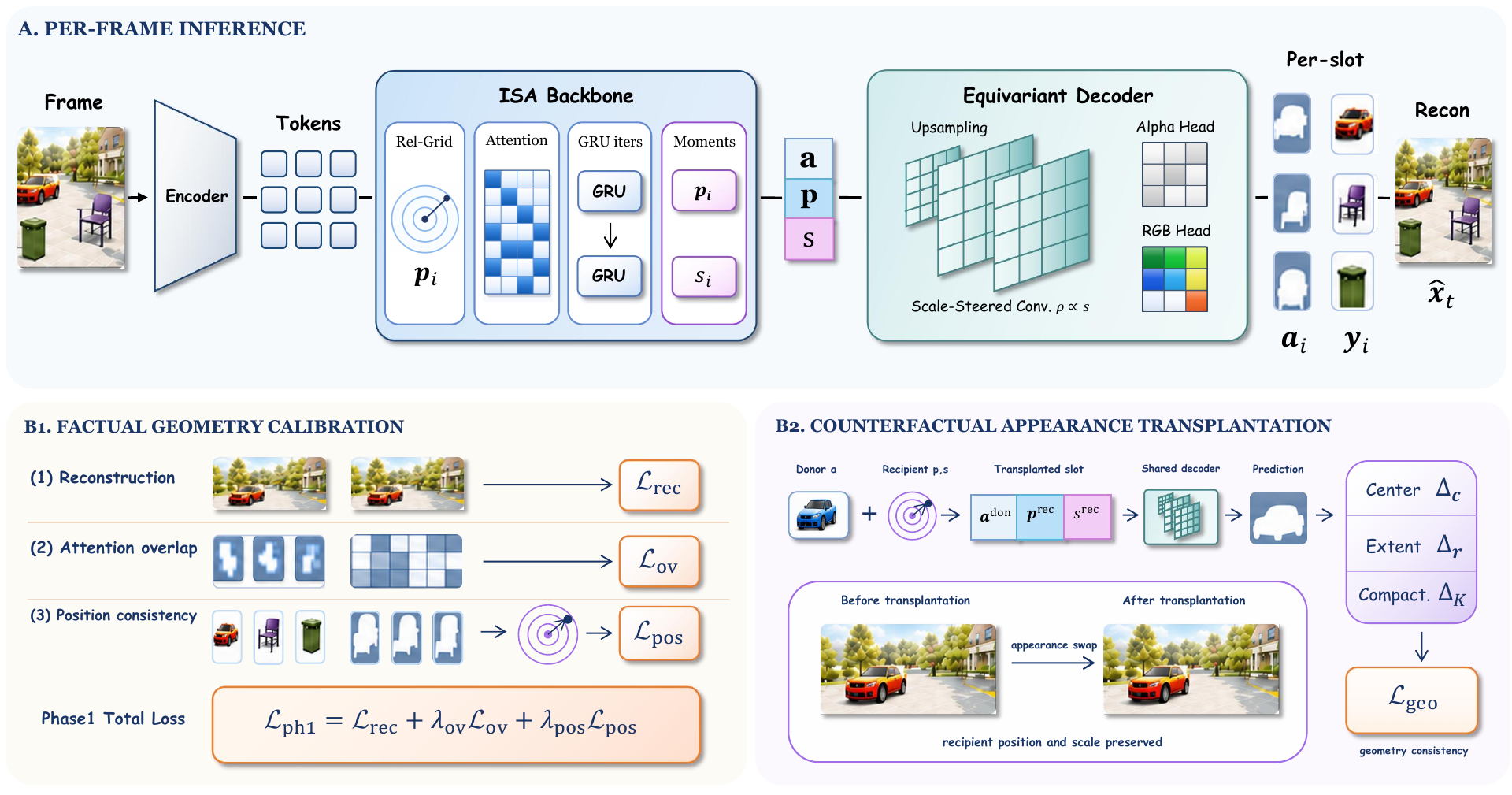}
    \caption{GeoCo-SAVi overview.
    \emph{A:} ISA infers per-slot appearance and geometry; a spatially
    equivariant decoder renders RGB and alpha.
    \emph{B1:} factual losses calibrate allocation and position.
    \emph{B2:} transplants align recipient center and extent while retaining
    donor compactness.}
    \label{fig:architecture}
\end{figure}

\begingroup
\addtolength{\abovedisplayskip}{2pt}
\addtolength{\belowdisplayskip}{2pt}
\addtolength{\abovedisplayshortskip}{2pt}
\addtolength{\belowdisplayshortskip}{2pt}
\subsection{Spatially Equivariant Decoder}
\label{sec:spatial-equivariant-decoder}

The desired interface is direct: changing $\mathbf p_i$ should translate the
rendered support, while changing $s_i$ should resize it around $\mathbf p_i$.
We consider two decoder implementations: scale-steered decoding on the
image-plane grid, and canonical-canvas decoding followed by a position--scale
warp. Both support predictable geometry-controlled editing
within our formulation. We describe the scale-steered implementation below;
Appendix~\ref{app:canonical-canvas} presents the canonical-canvas implementation and results.
Scale-steered decoding uses SAC's scale-adaptive sampling rule~\citep{sac}
within each slot renderer, with explicit $s_i$ controlling its sampling footprint.
For ideal continuous description, let
$\mathbf h:\R^2\rightarrow\R^{C_{\mathrm{in}}}$ be a feature field and
$\mathcal K\subset\mathbb Z^2$ a finite kernel support.  Its action and
layerwise sampling radius are
\begin{equation}
\begin{gathered}
    (\mathcal C_s\mathbf h)(\mathbf u)
    =\mbox{$\sum$}_{\boldsymbol\delta\in\mathcal K}
    \mathbf W_{\boldsymbol\delta}\,
    \mathbf h\!\left(\mathbf u+\rho_\ell(s)\boldsymbol\delta\right)
    +\mathbf b,\\[0.2em]
    \rho_\ell(s)=\Delta_\ell{s}/{s_{\mathrm{ref}}}\kappa_\ell,
    \\[0.2em]
    \kappa_\ell=\exp\!\left(\log \kappa_{\max}\tanh\psi_\ell\right).\\[0.1em]
\end{gathered}
    \label{eq:ssc}
\end{equation}
Here $\Delta_\ell$ is one normalized-grid pixel, $s_{\mathrm{ref}}>0$ is a global gauge, and $\kappa_\ell$ is an identity-initialized bounded gain. At $s=s_{\mathrm{ref}}$ and $\kappa_\ell=1$, the taps match ordinary stride-one convolution.

To state the requested scale response formally, let
$\mathcal T_{\mathbf p,k}$ rescale canvas coordinates by $k>0$ around
$\mathbf p$, and let $\mathcal D_{\mathbf p,k}$ apply the corresponding
transformation to a feature field:
\begin{equation}
\begin{aligned}
    \mathcal T_{\mathbf p,k}(\mathbf u)&=\mathbf p+k(\mathbf u-\mathbf p),\\
    (\mathcal D_{\mathbf p,k}\mathbf h)(\mathbf v)
    &=\mathbf h(\mathcal T_{\mathbf p,1/k}(\mathbf v)).
\end{aligned}
    \label{eq:scale-action}
\end{equation}
Because $\rho_\ell(ks)=k\rho_\ell(s)$, the continuous operators obey
\begin{equation}
    \mathcal C_{ks}\mathcal D_{\mathbf p,k}
    =\mathcal D_{\mathbf p,k}\mathcal C_s.
    \label{eq:ssc-commutation}
\end{equation}
Thus commanding $ks$ and then filtering is equivalent, in the ideal
continuous setting, to filtering at $s$ and rescaling the resulting field
around the same commanded slot center.
\begin{proposition}[Continuous raw-logit covariance]
\label{prop:raw-logit-equivariance}
Let $\ell_{\mathbf a,\mathbf p,s}$ be the ideal continuous per-slot alpha logit
generated from the slot-centric broadcast field by pointwise maps and
scale-steered layers with $\rho_\ell(ks)=k\rho_\ell(s)$.  With zero coordinate
stabilization and no boundary truncation, for every $\mathbf t\in\R^2$ and
$k>0$,
\begin{equation}
\begin{aligned}
\ell_{\mathbf a,\mathbf p+\mathbf t,s}(\mathbf u+\mathbf t)
    &=\ell_{\mathbf a,\mathbf p,s}(\mathbf u),\\
\ell_{\mathbf a,\mathbf p,ks}(\mathcal T_{\mathbf p,k}(\mathbf u))
    &=\ell_{\mathbf a,\mathbf p,s}(\mathbf u).
\end{aligned}
\label{eq:raw-logit-covariance}
\end{equation}
\end{proposition}
Thus the ideal raw-logit path is covariant to commanded translation and scale
(layerwise proof: Supplementary Material, Sec.~\ref{app:ssc-proof}), excluding cross-slot
full-softmax alpha.  Finite-grid discretization, boundaries, and the bounded
fixed-pixel RGB residual make equivariance approximate
(Sec.~\ref{sec:implementation-details}).

\subsection{Phase 1: Factual Geometry Calibration}

Let $m_i$ be slot $i$'s differentiable one-foreground (OneFG) support for a
detached factual background set (Eq.~\eqref{eq:onefg}).  For any support,
$\boldsymbol\mu$, $r$, $A$, and $K=A/(r^2+\epsilon)$ denote centroid, RMS
radius, coverage, and compactness (Sec.~\ref{sec:implementation-details}).

We first optimize the per-frame reconstruction loss
\begin{equation}
    \mathcal L_{\mathrm{rec}}
    =\operatorname{\mathbb{E}}_{c,\mathbf u}\left[
    \left(\widehat x_c(\mathbf u)-x_c(\mathbf u)\right)^2\right].
    \label{eq:reconstruction-loss}
\end{equation}
After warm-up, normalized overlap calibrates slot allocation:
\begin{equation}
    \mathcal L_{\mathrm{ov}}
    =\operatorname{\mathbb{E}}_{i\ne j}\left[\sum\nolimits_n w_i(\mathbf u_n)w_j(\mathbf u_n)\right].
    \label{eq:overlap-loss}
\end{equation}
Unit spatial mass penalizes weak duplicates (uniform slot: $1/L$) without
appearance-space repulsion. The position loss aligns each valid factual OneFG centroid with its slot
position:
\begin{equation}
    \mathcal L_{\mathrm{pos}}
    =\operatorname{\mathbb{E}}_{i\in\mathcal V}\left[
    \overline{\Huber}_{\delta_p}\!\left(
    \boldsymbol\mu(m_i)-\sg(\mathbf p_i)
    \right)\right],
    \label{eq:position-loss}
\end{equation}
Here $\overline{\Huber}_\delta$ averages coordinatewise Huber loss and
$\mathcal V$ indexes valid supports; $\sg$ denotes stop-gradient. Only the support path receives gradients;
$(\mathbf p_i,s_i)$ and validity are detached
(piecewise form: Supplementary Material, Sec.~\ref{app:huber-loss}).

\endgroup
\subsection{Phase 2: Counterfactual Appearance Transplantation}
\label{sec:appearance-transplant}

Factual alignment cannot constrain unseen appearance--geometry combinations. We transplant donor $j$ into recipient $i$ as
\begin{equation}
    \mathbf z_{i\leftarrow j}^{\cf}
    =[\mathbf a_j^{\don};\mathbf p_i^{\rec};s_i^{\rec}].
    \label{eq:counterfactual-slot}
\end{equation}
For an accepted pair, $m_i^{\rec}$ and $m_j^{\don}$ are factual supports, and
$m_{i\leftarrow j}^{\cf}$ is decoded against the recipient's factual background.
Pair construction and detached filtering are in Supplementary Material,
Sec.~\ref{app:appearance-pairing}.  The geometry residuals are
\begin{align}
    \boldsymbol\Delta_c^{ij}
    &=\boldsymbol\mu(m_{i\leftarrow j}^{\cf})-\sg(\mathbf p_i),
    \label{eq:geo-center}\\[2pt]
    \Delta_r^{ij}
    &=\log r(m_{i\leftarrow j}^{\cf})
    -\sg(\log r(m_i^{\rec})),
    \label{eq:geo-radius}\\[2pt]
    \Delta_K^{ij}
    &=\log K(m_{i\leftarrow j}^{\cf})
    -\sg(\log K(m_j^{\don})).
    \label{eq:geo-compactness}
\end{align}
With detached factual confidence $\omega_{ij}$, the objective is
\begin{equation}
\begin{aligned}
    \mathcal L_{\mathrm{geo}}
    =\mathbb E_{(i,j)\sim\omega}\!\Big[
    \mathcal L_{\mathrm{ctr}}\big(\boldsymbol\Delta_c^{ij}\big)
    +\Huber_{\delta_r}\big(\Delta_r^{ij}\big)+\lambda_K\Huber_{\delta_K}\big(\Delta_K^{ij}\big)\Big].
\end{aligned}
\label{eq:geometry-loss}
\end{equation}
With factual targets and recipient commands detached, $\mathcal L_{\mathrm{geo}}$
updates the decoder and donor-appearance path through counterfactual alpha.
Proposition~\ref{prop:raw-logit-equivariance} characterizes the geometric transformation of the donor raw-logit field under recipient commands; support consistency is encouraged by calibration and evaluated on the finite grid.  $\mathcal L_{\mathrm{geo}}$ aligns center and
RMS extent across appearances while matching donor compactness, hence
$A^{\mathrm{target}}_{i\leftarrow j}=\sg(K_j^{\don})\sg(r_i^{\rec})^2$;
Supplementary Material, Secs.~\ref{app:center-tail-losses} and~\ref{app:appearance-pairing} define the center penalty and coherence filter.

\subsection{Training Objective and Curriculum}
\label{sec:curriculum}

The single-frame training objective is
\begin{equation}
    \mathcal L_{\mathrm{train}}=\mathcal L_{\mathrm{rec}}
    +\lambda_{\mathrm{ov}}\mathcal L_{\mathrm{ov}}
    +\lambda_{\mathrm{pos}}\mathcal L_{\mathrm{pos}}
    +\lambda_{\mathrm{geo}}\mathcal L_{\mathrm{geo}}
    +\lambda_{\mathrm{gate}}\mathcal L_{\mathrm{gate}}
    +\lambda_{\mathrm{tail}}\mathcal L_{\mathrm{tail}}.
    \label{eq:core-objective}
\end{equation}
Here $\mathcal L_{\mathrm{gate}}$ penalizes the mean squared terminal gate (Eq.~\eqref{eq:terminal-readout}); the selector has a separate objective and optimizer in the Obj3D trajectory. On MOVi-C, the terminal readout, selector, and tail penalty are inactive; the results demonstrate geometry-controlled video editing without relying on these optional refinements.

We warm up reconstruction, ramp factual position and overlap, and add
counterfactual geometry. Video training additionally activates slot
initialization from preceding frames. The Obj3D and MOVi-C schedules, including reference-scale continuation, are in Supplementary Material, Secs.~\ref{app:training-details} and~\ref{app:movic}, respectively.
Loss weights are chosen from the observed loss and module-wise gradient
magnitudes, then rounded to simple values.

\subsection{Implementation Details}
\label{sec:implementation-details}

\noindent\textbf{Optional terminal appearance readout.}
On Obj3D, a bounded, zero-initialized gate refines appearance after the final
readout while leaving $\mathbf p_i$ and $s_i$ unchanged.  The returned
appearance enters both decoder branches.  Supplementary Material,
Sec.~\ref{app:terminal-readout} gives the complete residual gate and routing
definition.

\noindent\textbf{Discrete decoder realization.}
Following spatial broadcast~\citep{spatialbroadcast}, each slot is placed on a
slot-centric grid; progressive bilinear synthesis then alternates with
scale-steered convolutions.  Every learned spatial operation on the alpha path
is scale-steered.  RGB additionally uses a bounded, zero-initialized fixed-pixel
micro residual on Obj3D.  Sampling, padding, clipping, aliasing, and cross-slot
competition delimit finite-grid equivariance.  Supplementary Material,
Secs.~\ref{app:discrete-ssc}--\ref{app:decoder-topology} give the discrete
operator, heads, and topology.

\noindent\textbf{Differentiable geometry support.}
Geometry training uses the differentiable logits of a detached factual
background set $\mathcal B$:
\begin{equation}
\begin{gathered}
    \ell_{\mathcal B}(\mathbf u)
    =\logsumexp_{j\in\mathcal B}\ell_j(\mathbf u),\quad
    m_i(\mathbf u)
    =\sigma\!\left(\gamma[
    \ell_i(\mathbf u)-\ell_{\mathcal B}(\mathbf u)]\right),
    \quad \gamma=2.
\end{gathered}
\label{eq:onefg}
\end{equation}
This sharpened OneFG mask removes other foreground competitors.  Invisible
geometry remains outside the proxy.  For nonnegative $m$, define
\begin{equation}
\begin{gathered}
M_m=\mbox{$\sum$}_{\mathbf u}m(\mathbf u),\quad
\boldsymbol\mu(m)=
(M_m+\epsilon)^{-1}{\mbox{$\sum$}_{\mathbf u}m(\mathbf u)\mathbf u},\\[0.5em]
r(m)^2=({M_m+\epsilon})^{-1}{\mbox{$\sum$}_{\mathbf u}m(\mathbf u)
\|\mathbf u-\boldsymbol\mu(m)\|_2^2}+\epsilon,\\[0.5em]
A(m)=\mathbb{E}_{\mathbf u}[m(\mathbf u)],\quad
K(m)=(r(m)^2+\epsilon)^{-1}{A(m)}.
\end{gathered}
\label{eq:support-moments}
\end{equation}
Under uniform scaling, $r$ is first-order, $A$ second-order, and $K$
invariant (Supplementary Material, Sec.~\ref{app:moment-proof}).  OneFG provides the
differentiable training proxy; the main editing metrics use hard full-softmax ownership.

\paragraph{Optional selective alpha ownership.}
A jointly trained selector refines foreground/background ownership using
detached internal evidence. Its operating point is fixed on a calibration
panel from the training split. Supplementary Material,
Secs.~\ref{app:transfer-minimality}--\ref{app:teacher-details} give the ownership
transfer rule, training-only teacher, and selector details.

\paragraph{Video slot initialization.}
For video inference, a STATM-style~\citep{statm} Transformer initializes slots
from preceding slot states, followed by current-frame ISA refinement.
It is trained through frame reconstruction; its residual update is detailed
in Supplementary Material, Sec.~\ref{app:temporal-initializer}.

\section{Experiments}
\label{sec:experiments}

We evaluate three questions: (1) whether GeoCo-SAVi preserves factual
reconstruction and decomposition quality; (2) whether position and scale
behave as commands; and (3) whether their meaning remains stable across appearances.
Controlled single-frame Obj3D~\citep{gswm} isolates these effects, while
MOVi-C~\citep{kubric} tests them on complex video objects.

\subsection{Experimental Setup and Evaluation Design}

\paragraph{Shared protocol and three support roles.}
Training uses one NVIDIA A800 80GB GPU. \emph{Training support} consists
solely of model-produced decoder/attention supports, including differentiable
OneFG (Eq.~\eqref{eq:onefg}).
\emph{Reference support} supplies external centers and pre-edit identities:
model-independent RGB pseudo-labels
on Obj3D and official visible GT masks on MOVi-C.
\emph{Measurement support} is the decoded mask obtained by argmax over full-softmax alpha.
Editing F-scores measure adherence to prescribed position, scale, and
appearance-transplantation commands; reference-based metrics assess
alignment with scene geometry.
All-slot evaluations include valid non-background objects without
cross-method matching.
Geometric response metrics use objects within the canvas to avoid clipping-induced
changes in center and size; F1 uses valid edits without boundary filtering. Results without
boundary filtering appear in Tables~\ref{tab:edit-mask-f1}--\ref{tab:obj3d-new-curves}
and~\ref{tab:movic-new-edits}--\ref{tab:movic-new-curves}.

\begingroup
\setlength{\parskip}{9pt}
With decoded center $\hat c$ and reference center $c^\star$,
$E_{p\rightarrow\hat c}$, $E_{p\rightarrow c^\star}$, and
$E_{\hat c\rightarrow c^\star}$ denote mean L2 distances
on the $[-1,1]^2$ grid. $O_{\mathrm{attn}}$ denotes normalized attention overlap.
At fixed geometry $g=(p,s)$, $\sigma_{r\mid g}$ and $\sigma_{A\mid g}$ measure
population radius and coverage variation across appearances. Radius uses normalized-grid
RMS units, and coverage is the area fraction.
For an edited mask $M_e$, the target $M_t$ is obtained by applying the
commanded transformation to the corresponding factual decoded mask.
Editing F1 is $2|M_e\cap M_t|/(|M_e|+|M_t|)$.
Position, scale, and appearance-transfer scores are
$F_{\mathrm{pos}},F_{\mathrm{scl}},F_{\mathrm{app}}$, respectively.
All-slot empty edited outputs receive zero F1; center-based errors use
outputs with defined centers.
\looseness=-1\par

\paragraph{Controlled Obj3D.}
Obj3D contains $64\times64$ scenes with moving colored solids.  ISA,
GeoCo-SAVi, and same-protocol SlotAug~\citep{slotaug} use six 64-dimensional
appearance states and train from scratch for 70k updates with matched capacity
and sample exposure. Isotropic ISA uses reconstruction alone; GeoCo-SAVi follows
Sec.~\ref{sec:curriculum}; SlotAug retains its published objective and commands,
with command normalization delayed to a later hidden layer.  PSNR averages four inference initializations.
Table~\ref{tab:obj3d-main} reports reconstruction, position alignment,
fixed-geometry stability, and editing responses. External position errors
cover all 2,600 windows, with 9,255 shared RGB pseudo-label objects.

\begingroup
\setlength{\intextsep}{16pt plus 2pt minus 2pt}
\setlength{\floatsep}{18pt plus 2pt minus 2pt}
\setlength{\textfloatsep}{22pt plus 2pt minus 2pt}
\begin{table}[!ht]
\centering
\small
\papertablecaption{Single-frame Obj3D results. PSNR and $O_{\mathrm{attn}}$ cover the full
validation set; $E_{p\rightarrow\hat c}$ and editing use the metric-specific scopes
in Sec.~\ref{sec:experiments}. Reference errors use 9,255 shared RGB pseudo-label
objects. Spreads fix $s=0.2$, with sample counts in the supplement; editing F1 is in percent. Best values are bold.}
\label{tab:obj3d-main}
\renewcommand{\arraystretch}{1.30}
\setlength{\tabcolsep}{3pt}
\tablesubheading{Factual Reconstruction and Geometry}
\begin{tabularx}{\textwidth}{l*{7}{>{\centering\arraybackslash}X}}
\toprule
\rowcolor{tableheadgray}
\textbf{Method} & PSNR$\uparrow$ & $E_{p\rightarrow\hat c}\downarrow$ & $E_{p\rightarrow c^\star}\downarrow$ & $E_{\hat c\rightarrow c^\star}\downarrow$ & $O_{\mathrm{attn}}\downarrow$ & $\sigma_{r\mid g}\downarrow$ & $\sigma_{A\mid g}\downarrow$ \\
\midrule
ISA & 44.145 & 0.1049 & 0.1023 & 0.0420 & 0.0148 & 0.0440 & 0.0183 \\
SlotAug & 43.216 & 0.0476 & 0.0579 & \textbf{0.0358} & 0.0124 & 0.0520 & 0.0226 \\
GeoCo-SAVi & \textbf{44.210} & \textbf{0.0107} & \textbf{0.0436} & 0.0391 & \textbf{0.0007} & \textbf{0.0211} & \textbf{0.0086} \\
\bottomrule
\end{tabularx}
\vspace{6pt}

\tablesubheading{Counterfactual Editing Capability}
\begin{tabularx}{\textwidth}{l*{7}{>{\centering\arraybackslash}X}}
\toprule
\rowcolor{tableheadgray}
\textbf{Method} & $F_{\mathrm{pos}}\uparrow$ & $F_{\mathrm{scl}}\uparrow$ & $F_{\mathrm{app}}\uparrow$ & $E_{\Delta p}\downarrow$ & $\beta_r\to1$ & $\beta_A\to2$ & $D_{\mathrm{app}}\downarrow$ \\
\midrule
ISA & 86.35 & 62.84 & 75.73 & 0.0091 & 0.0939 & -0.0092 & 0.0442 \\
SlotAug & 90.78 & 82.19 & 78.20 & 0.0060 & 0.5412 & 1.0575 & 0.0321 \\
GeoCo-SAVi & \textbf{93.29} & \textbf{92.72} & \textbf{91.09} & \textbf{0.0051} & \textbf{1.0339} & \textbf{2.0225} & \textbf{0.0124} \\
\bottomrule
\end{tabularx}
\par\vspace{18pt}
    \centering
    \includegraphics[width=\textwidth]{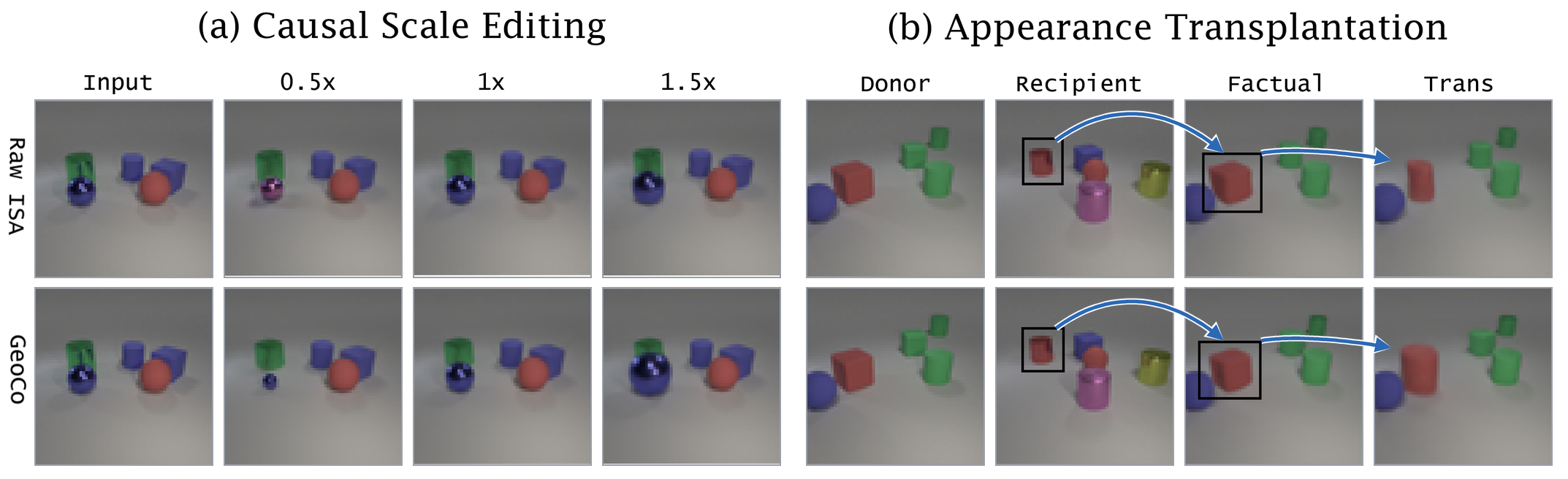}
    \captionsetup{position=bottom,aboveskip=12pt,belowskip=0pt}
    \captionof{figure}{Counterfactual control on complementary Obj3D examples.
    GeoCo-SAVi produces a predictable extent response to scale commands
    (\emph{left}) and places donor appearance at recipient geometry
    (\emph{right}); Raw ISA is comparatively scale-insensitive and its
    transplanted support drifts.}
    \label{fig:obj3d-counterfactual}
\par\vspace{18pt}
\end{table}

\noindent\textbf{MOVi-C.}
We use all 24 frames of the official 250-video validation set.  All methods
reconstruct $64\times64$ RGB from frozen DINOv2 features~\citep{dinov2}.
GeoCo-SAVi uses 11 slots, 64-dimensional appearance, and a 100k curriculum:
two-frame clips through 70k, then four-frame clips with a zero-residual temporal
initializer.  Same-protocol STATM-SAVi~\citep{statm} and
STAITUS~\citep{staitus} match capacity, updates, sample exposure, and common
settings.  STATM-SAVi uses a standard Slot Attention core with the published
slot-wise spatial and temporal Transformer. STAITUS retains its temporal
appearance predictor and active-set gate; baseline configurations are detailed
in Supplementary Sec.~\ref{app:movic}. Factual metrics cover all 250 videos.
External position evaluation uses all 1,543 GT tracks with fixed pre-edit
association. Complete-video editing uses 238 shared recipient/donor video
pairs; geometric response diagnostics use all eligible slots.
\looseness=-1\par

\titlespacing*{\subsection}{0pt}{14pt}{7pt}
\subsection{Factual Reconstruction and Decomposition}
\label{sec:factual-results}

\begin{table}[!ht]
\small
\setbox\ablationleftbox=\vbox\bgroup
\hsize=0.52\linewidth
\linewidth=\hsize
\parindent=0pt
\setlength{\parskip}{0pt}
\centering
\small
\captionsetup{type=table}\papertablecaption{MOVi-C segmentation on 250 equally weighted validation videos. Best values are bold.}
\label{tab:movic-factual}
\renewcommand{\arraystretch}{1.25}
\setlength{\tabcolsep}{2pt}
\begin{tabularx}{\linewidth}{@{}l>{\centering\arraybackslash}X>{\centering\arraybackslash}Xc>{\centering\arraybackslash}X@{}}
\toprule
\rowcolor{tableheadgray}
\textbf{Method} & \textbf{FG-ARI} & \textbf{ARI}
& \textbf{mIoU} & \textbf{mBO} \\
\midrule
STAITUS & 0.1550 & 0.1114 & 0.1242 & 0.1307 \\
STATM-SAVi & 0.1619 & 0.4052
& 0.2035 & 0.2085 \\
GeoCo-SAVi & \textbf{0.1935} & \textbf{0.4610}
& \textbf{0.2451} & \textbf{0.2507} \\
\bottomrule
\end{tabularx}
\par\kern0pt
\egroup
\setbox\ablationrightbox=\vbox\bgroup
\hsize=0.46\linewidth
\linewidth=\hsize
\parindent=0pt
\emergencystretch=1em
\baselineskip=12pt plus 3pt
\parfillskip=0pt plus 0.6\linewidth
PSNR (dB) is $10\log_{10}(1/\mathrm{MSE})$ per 24-frame video for RGB in $[0,1]$.
MOVi-C full-softmax alpha is resized to GT resolution before slotwise argmax.
FG-ARI, ARI, Hungarian mIoU, and mBO use all 24 frames per video,
averaged equally over 250 videos. Higher values indicate better object segmentation quality on the evaluated videos.
\par
\egroup
\setlength{\ablationheight}{\dimexpr\ht\ablationleftbox+\dp\ablationleftbox\relax}
\ifdim\dimexpr\ht\ablationrightbox+\dp\ablationrightbox\relax>\ablationheight
\setlength{\ablationheight}{\dimexpr\ht\ablationrightbox+\dp\ablationrightbox\relax}
\fi
\nointerlineskip
\noindent\vtop to\ablationheight{\vskip0pt\unvbox\ablationleftbox\vfil}\hfill
\vtop to\ablationheight{\vskip0pt\unvbox\ablationrightbox}
\par\smallskip

\end{table}

GeoCo-SAVi matches ISA reconstruction with about 95.5\% lower attention
overlap (Table~\ref{tab:obj3d-main}). On MOVi-C, it leads both references in PSNR
and all four segmentation metrics (Tables~\ref{tab:movic-factual}
and~\ref{tab:movic-control}).
\par

\subsection{Commanded Position and Scale}
\label{sec:commanded-results}

\begin{table}[!ht]
\centering
\small
\papertablecaption{MOVi-C reconstruction, geometry, and editing. Position errors use 1,543 GT tracks with
fixed pre-edit association; PSNR and $O_{\mathrm{attn}}$ use all 250 videos; spreads fix
$s=0.2$, with sample counts in the supplement. F1 (\%) pools all 24 frames per video and averages 238 shared
recipient/donor pairs. Other response metrics use objects within the canvas.
N/A denotes an unsupported geometry interface; best values are bold.}
\label{tab:movic-control}
\renewcommand{\arraystretch}{1.30}
\setlength{\tabcolsep}{3pt}
\tablesubheading{Factual Reconstruction and Geometry}
\begin{tabularx}{\textwidth}{l*{7}{>{\centering\arraybackslash}X}}
\toprule
\rowcolor{tableheadgray}
\textbf{Method} & PSNR$\uparrow$ & $E_{p\rightarrow\hat c}\downarrow$ & $E_{p\rightarrow c^\star}\downarrow$ & $E_{\hat c\rightarrow c^\star}\downarrow$ & $O_{\mathrm{attn}}\downarrow$ & $\sigma_{r\mid g}\downarrow$ & $\sigma_{A\mid g}\downarrow$ \\
\midrule
STATM-SAVi & 21.314 & N/A & N/A & N/A & 0.0045 & N/A & N/A \\
STAITUS & 19.960 & 0.0868 & 0.3417 & 0.3128 & \textbf{0.0008} & 0.1756 & 0.1439 \\
GeoCo-SAVi & \textbf{22.808} & \textbf{0.0276} & \textbf{0.2403} & \textbf{0.2341} & 0.0011 & \textbf{0.0503} & \textbf{0.0191} \\
\bottomrule
\end{tabularx}
\vspace{6pt}

\tablesubheading{Video Editing and Geometry Response}
\begin{tabularx}{\textwidth}{l*{7}{>{\centering\arraybackslash}X}}
\toprule
\rowcolor{tableheadgray}
\textbf{Method} & $F_{\mathrm{pos}}\uparrow$ & $F_{\mathrm{scl}}\uparrow$ & $F_{\mathrm{app}}\uparrow$ & $E_{\Delta p}\downarrow$ & $\beta_r\to1$ & $\beta_A\to2$ & $D_{\mathrm{app}}\downarrow$ \\
\midrule
STATM-SAVi & N/A & N/A & N/A & N/A & N/A & N/A & N/A \\
STAITUS & 75.34 & 72.14 & 45.15 & 0.0174 & 0.7322 & 1.2122 & 0.0585 \\
GeoCo-SAVi & \textbf{86.51} & \textbf{84.93} & \textbf{74.01} & \textbf{0.0067} & \textbf{0.8892} & \textbf{1.6178} & \textbf{0.0334} \\
\bottomrule
\end{tabularx}
\end{table}

\paragraph{Position.}
GeoCo-SAVi improves latent-to-decoded and latent-to-reference position alignment
on both datasets (Tables~\ref{tab:obj3d-main} and~\ref{tab:movic-control}).
On Obj3D, GeoCo-SAVi's decoded centers are closer to the RGB pseudo-label
centers than ISA's, while SlotAug has the lowest decoded-center reference error.
The error $E_{\Delta p}$ compares commanded and observed center
displacements in image-diagonal units.
\par
\newpage
\begin{samepage}
Video F1 edits the same slot through all 24 frames, sums intersections and mask
areas over time before applying the F1 ratio, and averages videos and commands equally.
Targets follow the prescribed transformations (Sec.~\ref{sec:experiments}).
STATM-SAVi lacks an explicit geometry-control interface.
\par\end{samepage}

\newpage
\paragraph{Scale.}
We use factors $0.5$, $0.75$, $1$, $1.25$, and $1.5$.
For $x\in\{r,A\}$, the fitted log-response slope $\beta_x$ has ideal target
1 for radius and 2 for coverage; editing F1 excludes the identity command.
GeoCo-SAVi improves scale-edit F1 on both datasets. Its Obj3D response closely follows
the ideal scaling law; MOVi-C retains
finite-grid and scene-dependent deviations
(Tables~\ref{tab:obj3d-main} and~\ref{tab:movic-control};
Fig.~\ref{fig:obj3d-counterfactual}).

\subsection{Appearance-Consistent Geometry}
\label{sec:appearance-results}

GeoCo-SAVi reduces appearance-transplant drift on both datasets with recipient geometry fixed.
$D_{\mathrm{app}}$ measures decoded-center displacement on the normalized grid.
At fixed $g$ with $s=0.2$, radius and coverage spreads are also lower,
indicating more consistent extent and area across appearances
(Tables~\ref{tab:obj3d-main} and~\ref{tab:movic-control}).
Complete-video appearance-transfer F1 is 74.01\% (STAITUS: 45.15\%).
These results and Fig.~\ref{fig:obj3d-counterfactual} support more consistent
geometric commands across donor appearances.

\begin{figure}[!htbp]
\centering
\begin{minipage}[c]{0.52\linewidth}
    \includegraphics[width=\linewidth]{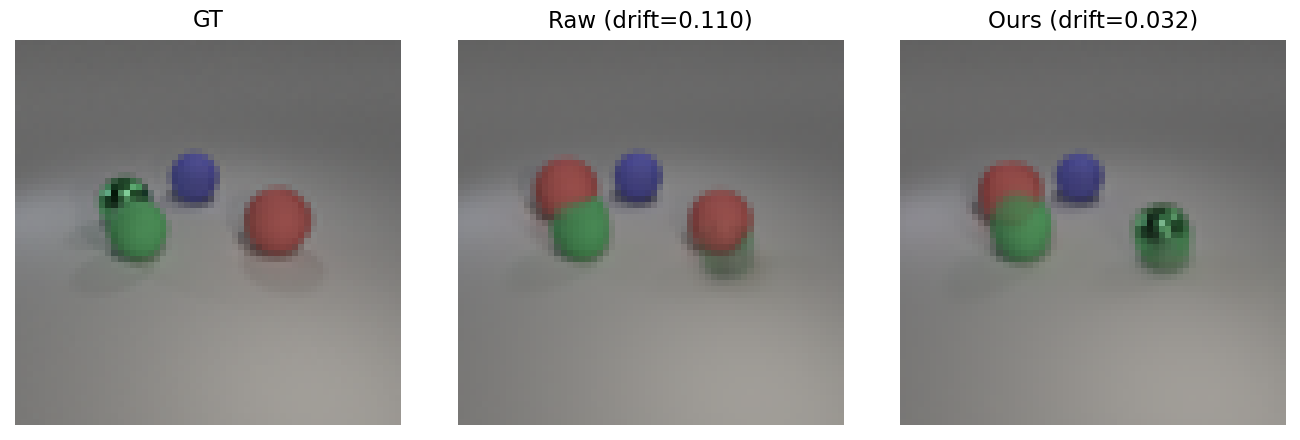}
    \caption{Appearance swap.}
    \label{fig:obj3d-appswap}
\end{minipage}\hfill
\begin{minipage}[c]{0.44\linewidth}
    The appearance-swap comparison in Fig.~\ref{fig:obj3d-appswap}
    shows that factual calibration reduces appearance-swap centroid drift.
    Raw ISA displaces the swapped object and leaves ghost artifacts; the
    calibrated diagnostic preserves its recipient location.  This 40k
    diagnostic follows a separate trajectory using reconstruction, position,
    and normalized-overlap losses.
\end{minipage}
\end{figure}

STAITUS's $O_{\mathrm{attn}}$ is computed only on active slots,
whereas the other methods include empty slots with uniform
attention; this convention favors lower STAITUS values. Scopes and
object-matched editing results with reference-mask-based RGB targets appear in Secs.~\ref{app:obj3d-evaluation}--\ref{app:movic}
(\ref{app:obj3d-edit-validation} and~\ref{app:movic-edit-validation}).
\par\endgroup
\endgroup

\clearpage
\subsection{Ablation Study}
\label{sec:ablations}

Seven controls assess decoder design, both calibration objectives, and optional refinements.

\par\addvspace{16pt}
\begingroup
\setlength{\parskip}{0pt}
\setbox\ablationleftbox=\vbox\bgroup
\hsize=0.44\linewidth
\linewidth=\hsize
\parindent=0pt
\setlength{\parskip}{0pt}
\centering
\small
\captionsetup{type=table}\papertablecaption{Obj3D component comparisons. Full denotes GeoCo-SAVi. Response metrics use the same within-canvas scope as Table~\ref{tab:obj3d-main}.}
\label{tab:ablation}
\renewcommand{\arraystretch}{1.04}
\setlength{\tabcolsep}{3pt}
\setlength{\aboverulesep}{1.5pt}
\setlength{\belowrulesep}{2pt}
\begin{tabularx}{\linewidth}{X>{\centering\arraybackslash}p{0.18\linewidth}>{\centering\arraybackslash}p{0.23\linewidth}}
\toprule
\rowcolor{tableheadgray}
\textbf{Diagnostic} & \textbf{Ablated} & \textbf{Full} \\
\midrule
\multicolumn{3}{@{}l}{\itshape Conventional decoder} \\
$\beta_r\to1$ & 0.7242 & \textbf{1.0339} \\
$\beta_A\to2$ & 1.3929 & \textbf{2.0225} \\
\midrule
\multicolumn{3}{@{}l}{\itshape Without $\mathcal L_{\mathrm{ov}}$} \\
$O_{\mathrm{attn}}\downarrow$ & 0.0077 & \textbf{0.0007} \\
\midrule
\multicolumn{3}{@{}l}{\itshape Without $\mathcal L_{\mathrm{pos}}$} \\
$E_{p\rightarrow\hat c}\downarrow$ & 0.0175 & \textbf{0.0107} \\
$E_{p\rightarrow c^\star}\downarrow$ & 0.0536 & \textbf{0.0436} \\
\midrule
\multicolumn{3}{@{}l}{\itshape Without $\mathcal L_{\mathrm{geo}}$} \\
$D_{\mathrm{app}}\downarrow$ & 0.0305 & \textbf{0.0124} \\
$\sigma_{r\mid g}\downarrow$ & 0.0256 & \textbf{0.0211} \\
$\sigma_{A\mid g}\downarrow$ & 0.0098 & \textbf{0.0086} \\
\midrule
\multicolumn{3}{@{}l}{\itshape Without terminal readout} \\
PSNR$\uparrow$ & 43.7399 & \textbf{44.2099} \\
$E_{\hat c\rightarrow c^\star}\downarrow$ & 0.0487 & \textbf{0.0391} \\
\midrule
\multicolumn{3}{@{}l}{\itshape Without selector} \\
PSNR$\uparrow$ & 44.2056 & \textbf{44.2099} \\
$E_{\hat c\rightarrow c^\star}\downarrow$ & 0.0451 & \textbf{0.0391} \\
\midrule
\multicolumn{3}{@{}l}{\itshape Without $\mathcal L_{\mathrm{tail}}$} \\
PSNR$\uparrow$ & \textbf{44.2995} & 44.2099 \\
$D_{\mathrm{app}}\downarrow$ & 0.0215 & \textbf{0.0124} \\
\bottomrule
\end{tabularx}
\par\kern0pt
\egroup
\setbox\ablationrightbox=\vbox\bgroup
\hsize=0.54\linewidth
\linewidth=\hsize
\parindent=0pt
\fontsize{11}{13.2}\selectfont
\setlength{\parskip}{6pt}
\emergencystretch=1em
\parfillskip=0pt plus 1fil

\noindent\textbf{Decoder design.}
The conventional decoder has radius and coverage slopes below their
ideal targets; the full model is closer to both. These slopes measure
responses to scale commands.
\looseness=-1\par\vfill

\noindent\textbf{Attention allocation.}
Without $\mathcal L_{\mathrm{ov}}$, factual attention overlap rises from
0.0007 to 0.0077, indicating that spatial normalization discourages duplicate
attention without a foreground/background gate.
\par\vfill

\noindent\textbf{Factual position calibration.}
Without $\mathcal L_{\mathrm{pos}}$, latent-to-decoded center error rises
from 0.0107 to 0.0175, while latent-to-reference error increases
from 0.0436 to 0.0536. Factual calibration improves command alignment
and slightly reduces reference-center error.
\par\vfill

\noindent\textbf{Counterfactual geometry.}
Without $\mathcal L_{\mathrm{geo}}$, transfer centroid drift increases from
0.0124 to 0.0305; radius and coverage spreads increase from
0.0211/0.0086 to 0.0256/0.0098, supporting its role in stabilizing
object centers and spatial extents across donor appearances.
\par\vfill

\noindent\textbf{Optional components.}
Terminal readout raises PSNR from 43.7399 to 44.2099 and lowers
decoded-center reference error from 0.0487 to 0.0391. The selector also
improves both metrics. The tail penalty trades a small PSNR decrease
(44.2995 to 44.2099) for lower appearance-transplant drift (0.0215 to 0.0124).
\par\egroup
\setlength{\ablationheight}{\dimexpr\ht\ablationleftbox+\dp\ablationleftbox\relax}
\ifdim\dimexpr\ht\ablationrightbox+\dp\ablationrightbox\relax>\ablationheight
  \setlength{\ablationheight}{\dimexpr\ht\ablationrightbox+\dp\ablationrightbox\relax}
\fi
\nointerlineskip
\noindent\vtop to \ablationheight{\vskip0pt\unvbox\ablationleftbox\vfil}\hfill
\vtop to \ablationheight{\vskip0pt\unvbox\ablationrightbox}
\endgroup
\par\addvspace{6pt}

All variants are evaluated at 70k updates using the same metric definitions and evaluation procedures as the full model. PSNR is averaged over four inference initializations. Component settings and training schedules are detailed in Supplementary Sec.~\ref{app:component-comparisons}.

\begingroup
\setlength{\parskip}{4pt}
\titlespacing*{\section}{0pt}{6pt}{3pt}
\section{Conclusion}
\label{sec:conclusion}

\noindent\textbf{Summary.}
GeoCo-SAVi connects explicit position--scale representations to decoded object geometry, enabling calibrated readout and direct editing with consistent semantics across appearances. On Obj3D, it matches ISA reconstruction with better geometry and editing. On MOVi-C, it improves reconstruction and grouping over same-protocol references and video editing over STAITUS.

\noindent\textbf{Limitations.}
Equivariance is established for continuous per-slot raw logits; discretization, clipping, and cross-slot softmax limit its finite-grid realization. Scale is isotropic, decoded support need not match physical object boundaries, and occlusion can impair geometric recovery. Video-training cost limits configuration and hyperparameter coverage.

\noindent\textbf{Future work.}
Priorities include broader validation, explicit depth and occlusion modeling, and combining calibrated image-plane representations with camera and scene geometry to recover real-world object trajectories for world-model learning and prediction under partial visibility.
\par\endgroup

\FloatBarrier
\addcontentsline{toc}{section}{References}
\begingroup
\small

\renewcommand{\bibfont}{\fontsize{10}{12.8}\selectfont}

\endgroup

\clearpage
\appendix
\linespread{1.05}\selectfont
\setlength{\parskip}{7pt plus 1pt minus 1pt}
\numberwithin{equation}{section}
\section*{Supplementary Material}
\addcontentsline{toc}{section}{Supplementary Material}
\addtocontents{toc}{\protect\setcounter{tocdepth}{1}}
\section{Continuous Equivariance and Scaling Proofs}
\label{app:equivariance-proofs}

\subsection{Scale-Steered Convolution}
\label{app:ssc-proof}

\begin{lemma}[Layer commutation]
\label{lem:ssc-commutation}
Let $\mathcal C_s$ be the continuous scale-steered layer in
Eq.~\eqref{eq:ssc}, acting on fields over $\R^2$, and assume
$\rho_\ell(ks)=k\rho_\ell(s)$ for every $k>0$.  Define translation by
$(\mathcal D_{\mathbf t}\mathbf h)(\mathbf v)
=\mathbf h(\mathbf v-\mathbf t)$.  With no boundary truncation,
\begin{equation}
    \mathcal C_{ks}\mathcal D_{\mathbf p,k}
    =\mathcal D_{\mathbf p,k}\mathcal C_s,
    \qquad
    \mathcal C_s\mathcal D_{\mathbf t}
    =\mathcal D_{\mathbf t}\mathcal C_s.
    \label{eq:layer-commutation}
\end{equation}
Every pointwise affine map or nonlinearity commutes with both actions as well.
\end{lemma}

\begin{proof}
For scale covariance, evaluate the left-hand side at
$\mathcal T_{\mathbf p,k}(\mathbf u)$:
\begin{equation}
\begin{aligned}
    &(\mathcal C_{ks}\mathcal D_{\mathbf p,k}\mathbf h)
    (\mathcal T_{\mathbf p,k}(\mathbf u))\\
    &\quad=\sum_{\boldsymbol\delta\in\mathcal K}
    \mathbf W_{\boldsymbol\delta}
    (\mathcal D_{\mathbf p,k}\mathbf h)
    (\mathcal T_{\mathbf p,k}(\mathbf u)
    +\rho_\ell(ks)\boldsymbol\delta)+\mathbf b\\
    &\quad=\sum_{\boldsymbol\delta\in\mathcal K}
    \mathbf W_{\boldsymbol\delta}\mathbf h\!\left(
    \mathbf u+\frac{\rho_\ell(ks)}{k}\boldsymbol\delta
    \right)+\mathbf b\\
    &\quad=\sum_{\boldsymbol\delta\in\mathcal K}
    \mathbf W_{\boldsymbol\delta}\mathbf h\!\left(
    \mathbf u+\rho_\ell(s)\boldsymbol\delta
    \right)+\mathbf b
    =(\mathcal C_s\mathbf h)(\mathbf u).
\end{aligned}
\end{equation}
The third line applies $\mathcal T_{\mathbf p,1/k}$ to the sampling
coordinate, and the fourth uses $\rho_\ell(ks)=k\rho_\ell(s)$.  By the
definition of $\mathcal D_{\mathbf p,k}$, the final expression equals
$(\mathcal D_{\mathbf p,k}\mathcal C_s\mathbf h)
(\mathcal T_{\mathbf p,k}(\mathbf u))$.  Translation follows from
\begin{equation}
    (\mathcal C_s\mathcal D_{\mathbf t}\mathbf h)(\mathbf u+\mathbf t)
    =\sum_{\boldsymbol\delta\in\mathcal K}
    \mathbf W_{\boldsymbol\delta}
    \mathbf h(\mathbf u+\rho_\ell(s)\boldsymbol\delta)+\mathbf b
    =(\mathcal C_s\mathbf h)(\mathbf u).
\end{equation}
Finally, for a pointwise map $\phi$ and spatial transformation $\mathcal T$ inducing $\mathcal D$,
$\phi(\mathcal D\mathbf h)(\mathbf v)
=\phi(\mathbf h(\mathcal T^{-1}\mathbf v))
=(\mathcal D\phi(\mathbf h))(\mathbf v)$.
\end{proof}

\begin{theorem}[Per-slot raw-logit covariance]
\label{thm:raw-logit-covariance}
Set the coordinate stabilizer to zero.  Let the initial per-slot field be
\begin{equation}
    \mathbf h_{\mathbf a,\mathbf p,s}^{(0)}(\mathbf u)
    =\operatorname{Broadcast}(\mathbf a)
    +\operatorname{PE}\!\left(
      \frac{\mathbf u-\mathbf p}{a_s s}\right),
    \label{eq:continuous-broadcast}
\end{equation}
and let $\ell_{\mathbf a,\mathbf p,s}$ be obtained from this field by any
finite composition of pointwise maps and scale-steered layers satisfying the
assumptions of Lemma~\ref{lem:ssc-commutation}, followed by a pointwise
alpha-logit head.  Then
\begin{equation}
\begin{aligned}
\ell_{\mathbf a,\mathbf p+\mathbf t,s}(\mathbf u+\mathbf t)
    &=\ell_{\mathbf a,\mathbf p,s}(\mathbf u),\\
\ell_{\mathbf a,\mathbf p,ks}(\mathcal T_{\mathbf p,k}(\mathbf u))
    &=\ell_{\mathbf a,\mathbf p,s}(\mathbf u)
\end{aligned}
\label{eq:decoder-raw-logit-covariance}
\end{equation}
for every $\mathbf t\in\R^2$ and $k>0$.
\end{theorem}

\begin{proof}
The initial field is translation covariant because
\[
\frac{(\mathbf u+\mathbf t)-(\mathbf p+\mathbf t)}{a_s s}
=\frac{\mathbf u-\mathbf p}{a_s s},
\]
and scale covariant because
\[
\frac{\mathcal T_{\mathbf p,k}(\mathbf u)-\mathbf p}{a_sks}
=\frac{\mathbf u-\mathbf p}{a_ss}.
\]
Thus Eq.~\eqref{eq:decoder-raw-logit-covariance} holds at layer zero.
Lemma~\ref{lem:ssc-commutation} preserves the corresponding identity after
each scale-steered or pointwise layer.  Induction over the finite decoder path,
including the pointwise alpha head, proves the result and hence
Proposition~\ref{prop:raw-logit-equivariance} in the main paper.
\end{proof}

\begin{remark}[Scope]
The theorem is exact only for the ideal continuous per-slot raw-logit
operator.  Finite-grid interpolation, padding, clipping, nonzero numerical
stabilization, and boundary truncation are excluded.  If all slot logits,
including background logits, undergo the same global action, pointwise
softmax also covaries.  Under an intervention that changes only one slot,
the other logits remain fixed and cross-slot full-softmax competition
generally prevents exact alpha equivariance.
\end{remark}

\subsection{Support Moments}
\label{app:moment-proof}

With $\epsilon=0$ and finite, positive mass and radius, assume an isolated continuous support obeys
\begin{equation}
    m_{ks}(\mathcal T_{\mathbf p,k}(\mathbf u))=m_s(\mathbf u).
\end{equation}
The substitution $\mathbf v=\mathcal T_{\mathbf p,k}(\mathbf u)$ has area element $d\mathbf v=k^2d\mathbf u$. It follows that
\begin{equation}
    \boldsymbol\mu(ks)
    =\mathbf p+k[\boldsymbol\mu(s)-\mathbf p],
    \qquad r(ks)=kr(s),
\end{equation}
and $A(ks)=k^2A(s)$ whenever the support is not clipped by the canvas. Consequently,
\begin{equation}
    K(ks)=\frac{k^2A(s)}{k^2r(s)^2}=K(s).
\end{equation}
Thus Eq.~\eqref{eq:geometry-loss} jointly aligns radius and compactness.

\subsection{Minimal Owner-to-Background Transfer}
\label{app:transfer-minimality}

For an accepted pixel, let $d_i=\ell_i-\ell_{\mathrm{bg}}$ and
$d_{\mathrm{alt}}=\ell_{\mathrm{alt}}-\ell_{\mathrm{bg}}$ denote the owner and
strongest alternative foreground gaps.  The minimum symmetric transfer needed
for background to win by margin $m_{\mathrm{bg}}$ is
\begin{equation}
    \Delta_{\mathrm{req}}(x)=\max\!\big\{0,\,
    (d_i+m_{\mathrm{bg}})/2,\,
    d_{\mathrm{alt}}+m_{\mathrm{bg}}\big\}.
    \label{eq:required-transfer-main}
\end{equation}
Its capped update is
\begin{equation}
\begin{aligned}
    \Delta_{\mathrm{tr}}&=\min(\Delta_{\mathrm{req}},\Delta_{\max}),
    &\ell_i'&=\ell_i-\Delta_{\mathrm{tr}},
    &\ell_{\mathrm{bg}}'&=\ell_{\mathrm{bg}}+\Delta_{\mathrm{tr}}.
\end{aligned}
\label{eq:logit-transfer-main}
\end{equation}
The action is zero outside accepted pixels.  Its threshold and cap are fixed
on a calibration panel drawn only from the training split.

Under the symmetric action $\ell_i'=\ell_i-\Delta_{\mathrm{tr}}$ and
$\ell_{\mathrm{bg}}'=\ell_{\mathrm{bg}}+\Delta_{\mathrm{tr}}$, background beats
the current owner by margin $m_{\mathrm{bg}}$ iff
\begin{equation}
\begin{aligned}
    \ell_{\mathrm{bg}}+\Delta_{\mathrm{tr}}
    &\ge \ell_i-\Delta_{\mathrm{tr}}+m_{\mathrm{bg}},\\
    &\Longleftrightarrow\quad
    \Delta_{\mathrm{tr}}
    \ge\frac{\ell_i-\ell_{\mathrm{bg}}+m_{\mathrm{bg}}}{2}.
\end{aligned}
\end{equation}
It beats the strongest untouched foreground competitor iff
\begin{equation}
\begin{aligned}
    \ell_{\mathrm{bg}}+\Delta_{\mathrm{tr}}
    &\ge \ell_{\mathrm{alt}}+m_{\mathrm{bg}},\\
    &\Longleftrightarrow\quad
    \Delta_{\mathrm{tr}}
    \ge \ell_{\mathrm{alt}}-\ell_{\mathrm{bg}}+m_{\mathrm{bg}}.
\end{aligned}
\end{equation}
Together with nonnegativity, Eq.~\eqref{eq:required-transfer-main} is the
smallest symmetric transfer satisfying both margins.
Equation~\eqref{eq:logit-transfer-main} applies its capped value and satisfies
them whenever the cap does not bind.

\section{Teacher--Selector Details}
\label{app:teacher-details}

\subsection{Background and Deployment Domain}

Let $\partial\Omega$ be a narrow image border. The background slot is selected from detached alpha by
\begin{equation}
    i_{\mathrm{bg}}=\arg\max_i\left[
    \frac{1}{|\partial\Omega|}\sum_{x\in\partial\Omega}\alpha_i(x)
    +\lambda_{\mathrm{mass}}\frac{1}{|\Omega|}
    \sum_{x\in\Omega}\alpha_i(x)
    \right].
    \label{eq:bg-selection}
\end{equation}
The deployment domain contains only foreground-owned pixels whose owner has sufficient hard area and peak alpha and whose required transfer does not exceed a labeling limit. These conditions are detached and are identical when producing teacher labels and applying the trained selector.

\subsection{High-Precision Tri-State Evidence}

For an owner $i$ and background $i_{\mathrm{bg}}$, define the individual renderer errors
\begin{align}
    E_i(x)&=\frac{1}{3}\|\mathbf y_i(x)-\mathbf x(x)\|_2^2,\\
    E_{\mathrm{bg}}(x)&=\frac{1}{3}
    \|\mathbf y_{i_{\mathrm{bg}}}(x)-\mathbf x(x)\|_2^2.
\end{align}
Let $E_{\mathrm{cur}}(x)$ and $E_{\mathrm{act}}(x)$ be the composite errors before and after the analytic transfer. We use
\begin{align}
G_{\mathrm{act}}(x)
&=\frac{E_{\mathrm{cur}}(x)-E_{\mathrm{act}}(x)}
{E_{\mathrm{cur}}(x)+E_{\mathrm{act}}(x)+\epsilon},
\label{eq:teacher-gain}\\
G_{\mathrm{ind}}(x)
&=\frac{E_i(x)-E_{\mathrm{bg}}(x)}
{E_i(x)+E_{\mathrm{bg}}(x)+\epsilon},
\label{eq:individual-advantage}
\end{align}
where positive $G_{\mathrm{ind}}$ favors background. Pixels for which the action is demonstrably harmful or the foreground renderer is substantially better form the protected set,
\begin{equation}
    c_-(x)=\mathbb I[G_{\mathrm{act}}(x)\le \theta_{\mathrm{act}}^-]
    \lor \mathbb I[G_{\mathrm{ind}}(x)\le\theta_{\mathrm{ind}}^-].
    \label{eq:protected-teacher}
\end{equation}

The positive teacher combines two complementary high-precision branches.  The
first, $c^+_{\mathrm{act}}$, requires a positive counterfactual action gain and
conservative attention/alpha-density checks.  The second,
$c^+_{\mathrm{local}}$, uses a local target-image background prototype to
break the feedback loop in which a pixel captured by foreground never trains
the background renderer.  Let $i(u)=\arg\max_k\alpha_k(u)$ denote the decoded
hard owner at pixel $u$.  For a neighborhood $\mathcal N(x)$,
\begin{equation}
    \boldsymbol\mu_{\mathrm{local}}(x)=
    \frac{\sum_{u\in\mathcal N(x)}
    \mathbb I[i(u)=i_{\mathrm{bg}}]\mathbf x(u)}
    {\sum_{u\in\mathcal N(x)}
    \mathbb I[i(u)=i_{\mathrm{bg}}]+\epsilon},
\end{equation}
\begin{equation}
    E_{\mathrm{local}}(x)=\frac{1}{3}
    \|\mathbf x(x)-\boldsymbol\mu_{\mathrm{local}}(x)\|_2^2.
\end{equation}
The branch requires enough background-owned neighbors, a low normalized local
error, and a feasible transfer magnitude.  The raw positive confidence and
conflict-resolved label are
\begin{equation}
\begin{aligned}
    c_+(x)&=\max\{c_{\mathrm{act}}^+(x),c_{\mathrm{local}}^+(x)\},\\
    \widetilde c_+(x)&=c_+(x)[1-c_-(x)].
\end{aligned}
    \label{eq:teacher-tristate}
\end{equation}

\subsection{Selector Features and Ranking Loss}

The raw pointwise evidence has 139 dimensions: eight ownership/logit/density scalars, a 64-dimensional owner feature, a 64-dimensional owner--background feature difference, and a three-dimensional equivariant RGB prediction difference. One shared $1\times1$ projection maps the owner feature to eight dimensions and, with the same weights and canceled bias, maps the owner--background difference to eight dimensions, giving a 27-dimensional selector input. With SiLU activation $\phi$,
\begin{equation}
\begin{aligned}
    \mathbf r_1(x)&=\phi(\mathbf W_1\mathbf f(x)+\mathbf b_1),\\
    \mathbf r_2(x)&=\phi(\mathbf W_2\mathbf r_1(x)+\mathbf b_2),\\
    \pi_{\mathrm{sel}}(x)&=\sigma(\mathbf w_3^\top\mathbf r_2(x)+b_3).
\end{aligned}
\label{eq:selector}
\end{equation}
Hidden widths 32 and 16 yield 1,961 trainable parameters.

Let $\nu(x)$ be the selector logit. Class losses are macro-averaged over nonempty scene--owner groups. For groups with both positive and protected pixels, soft extrema
\begin{align}
    \nu_+^{\min}
    &=-T_{\mathrm{rank}}\log
    \frac{\sum_x\widetilde c_+(x)e^{-\nu(x)/T_{\mathrm{rank}}}}
    {\sum_x\widetilde c_+(x)+\epsilon},\\
    \nu_-^{\max}
    &=T_{\mathrm{rank}}\log
    \frac{\sum_x c_-(x)e^{\nu(x)/T_{\mathrm{rank}}}}
    {\sum_x c_-(x)+\epsilon}
\end{align}
define
\begin{equation}
    \mathcal L_{\mathrm{rank}}
    =T_{\mathrm{rank}}\operatorname{softplus}\!\left(
    \frac{m_{\mathrm{rank}}+\nu_-^{\max}-\nu_+^{\min}}
    {T_{\mathrm{rank}}}
    \right).
\end{equation}
Ranking averages only these groups, with zero loss if none qualify. Unknown pixels contribute zero selector gradient.
With group-balanced positive and protected logistic losses $\mathcal L_+$ and $\mathcal L_-$, the complete objective is
\begin{equation}
    \mathcal L_{\mathrm{sel}}
    =\mathcal L_+ +\lambda_-\mathcal L_-
    +\lambda_{\mathrm{rank}}\mathcal L_{\mathrm{rank}}.
    \label{eq:selector-loss}
\end{equation}

\section{Implementation and Curriculum Details}
\begingroup
\setlength{\parskip}{3pt plus 1pt}
\label{app:training-details}

All experiments reported in the main paper use one NVIDIA A800 80GB GPU.

\subsection{Terminal Appearance Readout}
\label{app:terminal-readout}

ISA's final value aggregate is converted by the shared recurrent update into
an appearance proposal $\widetilde{\mathbf a}_i$.  The bounded residual readout is
\begin{equation}
\begin{aligned}
    \mathbf a_i^{\mathrm{out}}
    &=\mathbf a_i+\mathbf g_i\odot
    (\widetilde{\mathbf a}_i-\mathbf a_i),\\
    \mathbf g_i
    &=g_{\mathrm{read}}\tanh\!\left(
    \mathbf W_g\operatorname{LN}
    [\mathbf a_i;\widetilde{\mathbf a}_i-\mathbf a_i]+\mathbf b_g\right).
\end{aligned}
    \label{eq:terminal-readout}
\end{equation}
The bounded gate $\mathbf g_i$ is initialized to zero, giving an identity appearance update.
The returned $\mathbf a_i^{\mathrm{out}}$ enters both decoder branches and is
subsequently abbreviated as $\mathbf a_i$.  It refines the appearance representation used for RGB reconstruction
without directly updating $\mathbf p_i$ or $s_i$.

\subsection{Huber Loss}
\label{app:huber-loss}

For a scalar residual $e$ and threshold $\delta>0$, the notation used in
the main paper is
\begin{equation}
    \Huber_\delta(e)=
    \begin{cases}
        \tfrac{1}{2}e^2, & |e|\le\delta,\\
        \delta\left(|e|-\tfrac{1}{2}\delta\right),
        & |e|>\delta.
    \end{cases}
    \label{eq:huber-definition}
\end{equation}
For $\mathbf e\in\R^d$, the overlined form denotes the coordinate mean,
\begin{equation}
    \overline{\Huber}_\delta(\mathbf e)
    =\frac{1}{d}\sum_{q=1}^{d}\Huber_\delta(e_q).
    \label{eq:huber-vector}
\end{equation}

\subsection{Center and Tail Penalties}
\label{app:center-tail-losses}

For the center residual $\boldsymbol\Delta=\boldsymbol\Delta_c^{ij}$ in Eq.~\eqref{eq:geometry-loss}, recipient scale $s$, and global training step $t$, define
\begin{equation}
\mathcal L_{\mathrm{ctr}}(\boldsymbol\Delta;s,t)=
\begin{cases}
\frac12\sum_{d\in\{x,y\}}\Huber_{\delta_c}(\Delta_d),
&\text{coordinate-wise},\\[3pt]
\Huber_{\delta_c}\!\left(\frac{\|\boldsymbol\Delta\|_2}{s+\epsilon}\right),
&\text{scale-normalized}.
\end{cases}
\end{equation}
MOVi-C uses the coordinate-wise form throughout geometry training. Obj3D uses the coordinate-wise form before step 50,000 and the scale-normalized form thereafter, with $\delta_c=0.05$. Recipient commands remain detached. The main text abbreviates this penalty as $\mathcal L_{\mathrm{ctr}}(\boldsymbol\Delta_c^{ij})$, suppressing its scale and training-step dependence; the factual position loss is unchanged.

The tail penalty acts on gamma-2 OneFG support $q(x)=\operatorname{sigmoid}(2z(x))$, where $z$ is the target logit minus the aggregated background logit. For pixel grid $\Omega$, let
\begin{equation}
T(q)=-\frac{1}{|\Omega|}\sum_{x\in\Omega}
\sg\!\left[\mathbf 1\{q(x)<\tau\}\right]\log(1-q(x)),
\qquad \tau=0.25.
\end{equation}
The implementation evaluates $-\log(1-q)$ stably as $\operatorname{softplus}(2z)$. The detached threshold selects low-confidence pixels only; no positive-target term sharpens the remaining pixels. We combine factual and counterfactual penalties as
\begin{equation}
\mathcal L_{\mathrm{tail}}=\tfrac12\mathbb E_{\mathrm{FG}}[T(q^{\mathrm{fact}})]
+\tfrac12\mathbb E_{(i,j)\sim\omega}[T(q^{\mathrm{cf}}_{i\leftarrow j})].
\end{equation}
The factual mean uses accepted foreground slots with a nonempty background set; the counterfactual mean uses the detached pair-confidence weights of the geometry objective and the recipient's factual background logits. Spatial means are taken over the full canvas. Empty eligible populations contribute zero. Obj3D activates $\lambda_{\mathrm{tail}}=10^{-3}$ at step 60,000 without an additional ramp; MOVi-C uses $\lambda_{\mathrm{tail}}=0$.

\subsection{Discrete Scale-Steered Convolution}
\label{app:discrete-ssc}

Let $\mathbf H\in\R^{C_{\mathrm{in}}\times H_\ell\times W_\ell}$ be a discrete feature map and let $\mathbf q\in\mathbb Z^2$ index an output pixel. In feature-pixel units, the commanded tap dilation is
\begin{equation}
    d_\ell(s)=\frac{s}{s_{\mathrm{ref}}}\kappa_\ell.
    \label{eq:discrete-dilation}
\end{equation}
The implemented scale-steered cross-correlation is
\begin{equation}
\begin{aligned}
    (\widehat{\mathcal C}_s\mathbf H)_o[\mathbf q]
    ={}&b_o+\sum_{c=1}^{C_{\mathrm{in}}}
    \sum_{\boldsymbol\delta\in\mathcal K}
    W_{o,c,\boldsymbol\delta}
    \operatorname{Bilin}\!\left(
    \mathbf H_c,\mathbf q+d_\ell(s)\boldsymbol\delta
    \right).
\end{aligned}
\label{eq:discrete-ssc}
\end{equation}
Here $\mathbf W=[W_{o,c,\boldsymbol\delta}]$ is the learned kernel tensor and $\operatorname{Bilin}$ is bilinear sampling on the current feature grid. Pixel centers are converted to normalized $[-1,1]^2$ coordinates with aligned corner centers; samples outside the finite canvas use zero padding. When $s=s_{\mathrm{ref}}$ and $\kappa_\ell=1$, all taps are integer-aligned and Eq.~\eqref{eq:discrete-ssc} reduces to the corresponding ordinary stride-one convolution up to floating-point interpolation error.

Each rendered slot retains its own scalar $s$. The implementation constructs all $|\mathcal K|$ analytic tap locations directly from Eq.~\eqref{eq:discrete-dilation}; no network predicts deformable offsets. The taps are evaluated jointly by vectorized bilinear sampling, followed by the same channel contraction as a conventional convolution. This implementation and the reference per-tap implementation are algebraically identical, and gradients propagate through the sampler to the input field, scale command, and bounded layer gain.

\subsection{Decoder Topology}
\label{app:decoder-topology}

Following spatial broadcast~\citep{spatialbroadcast}, slot $i$ induces
\begin{equation}
    \mathbf h_i^{(0)}(\mathbf u)
    =\operatorname{Broadcast}(\mathbf a_i)
    +\operatorname{PE}\!\left(
    \frac{\mathbf u-\mathbf p_i}{a_s s_i+\epsilon}\right).
    \label{eq:broadcast-field}
\end{equation}
A learned pointwise embedding $\operatorname{PE}$ is injected once, after
which progressive bilinear synthesis alternates with bilinearly sampled
scale-steered convolutions.  The pointwise alpha head is
\begin{equation}
    \ell_i(\mathbf u)=\mathbf w_\alpha^\top
    \mathbf h_i^{\mathrm{eq}}(\mathbf u)+b_\alpha,
    \label{eq:alpha-head}
\end{equation}
so every learned spatial operation on its path is scale-steered.  RGB uses
a pointwise fusion of $\mathbf h_i^{\mathrm{eq}}$ with broadcast appearance
to form $\mathbf h_i^{\mathrm{rgb}}$, followed by
\begin{equation}
    \mathbf y_i(\mathbf u)=\sigma\!\left(
    \mathbf y_i^{\mathrm{eq}}(\mathbf u)
    +\eta_{\max}\tanh\!\left[
    \mathcal C_{3\times3}^{\mathrm{micro}}
    \mathbf h_i^{\mathrm{rgb}}(\mathbf u)\right]\right),
    \label{eq:rgb-head}
\end{equation}
where the fixed-pixel micro convolution is zero-initialized and the residual
is bounded by $\eta_{\max}$.  We enable it for Obj3D RGB; being bounded, it
refines appearance while leaving the alpha and geometry paths untouched.  Here
$\mathbf h_i^{\mathrm{eq}}$ is the final equivariant feature field and
$\mathbf y_i^{\mathrm{eq}}$ its scale-steered RGB-logit readout.  Sampling,
padding, clipping, aliasing, and cross-slot competition are the finite-grid
approximation boundaries.

The Obj3D instantiation broadcasts each slot to $16\times16$. It applies bilinear $2\times$ synthesis followed by a $5\times5$ scale-steered convolution at $32\times32$, another bilinear synthesis followed by a $3\times3$ scale-steered convolution at $64\times64$, and three $3\times3$ scale-steered blocks at the final resolution. The alpha head is $1\times1$. The RGB path pointwise-fuses the equivariant field with a spatially constant appearance style before its scale-steered head; only its bounded, zero-initialized micro branch uses a fixed-pixel $3\times3$ convolution. Progressive synthesis avoids the checkerboard phase preference of zero-insertion while retaining learned scale-dependent aggregation after every resolution change.

The global reference scale is continued during cold start without breaking the homogeneity of Eq.~\eqref{eq:ssc}. Abbreviate $s_c=s_{\mathrm{cold}}$ and $s_r=s_{\mathrm{ref}}$. With update $t$,
\begin{equation}
    s_{\mathrm{ref}}^{\mathrm{eff}}(t)=
    \begin{cases}
    s_c,&0\le t<2{,}000,\\
    s_c^{1-\beta_t}s_r^{\beta_t},
    &2{,}000\le t<10{,}000,\\
    s_r,&t\ge10{,}000,
    \end{cases}
\end{equation}
where $\beta_t=\operatorname{clip}_{[0,1]}((t-1999)/8000)$,
$s_{\mathrm{cold}}$ is the RMS scale of uniform attention on the initial grid,
and $s_{\mathrm{ref}}$ is the final operator gauge. At every fixed update $t$,
$\rho_{\ell,t}(ks)=k\rho_{\ell,t}(s)$ still holds for every layer $\ell$.

\subsection{Decoder Implementations}
\label{app:canonical-canvas}

We evaluate scale-steered and canonical-canvas slot decoding within
the GeoCo-SAVi geometry-control formulation. The canonical-canvas implementation
follows the local-patch transformation approach of AIR and SPACE~\citep{air,space}.
The decoder produces RGB and alpha logits on a fixed $64\times64$ canvas
at $p=0$ and $s_0=0.2$, using ordinary convolutions of the same widths and
kernel sizes. A final bilinear warp samples at $(\mathbf x-p)s_0/s$,
with aligned corners; out-of-canvas RGB and alpha logits are filled with
0 and $-20$, respectively. RGB sigmoid and slotwise alpha normalization
follow the warp. Both use factual and counterfactual calibration.

\begin{table}[htbp]
\centering
\small
\papertablecaption{Obj3D decoder implementations within GeoCo-SAVi at 70k updates. $\mathrm{PSNR}_{\mathrm{pos/scl/app}}$ measures full-image agreement with independent RGB edit targets (Sec.~\ref{app:obj3d-edit-validation}). F-scores are percentages; PSNR is in dB.}
\label{tab:canonical-canvas}
\setlength{\tabcolsep}{1pt}
\renewcommand{\arraystretch}{1.08}
\begin{tabularx}{\textwidth}{l>{\hsize=.85\hsize\linewidth=\hsize\centering\arraybackslash}X>{\hsize=1.2\hsize\linewidth=\hsize\centering\arraybackslash}X>{\hsize=1.2\hsize\linewidth=\hsize\centering\arraybackslash}X>{\hsize=1.2\hsize\linewidth=\hsize\centering\arraybackslash}X>{\hsize=.85\hsize\linewidth=\hsize\centering\arraybackslash}X>{\hsize=.85\hsize\linewidth=\hsize\centering\arraybackslash}X>{\hsize=.85\hsize\linewidth=\hsize\centering\arraybackslash}X}
\toprule
\rowcolor{tableheadgray}
Decoder & PSNR$\uparrow$ & $\mathrm{PSNR}_{\mathrm{pos}}\uparrow$ & $\mathrm{PSNR}_{\mathrm{scl}}\uparrow$ & $\mathrm{PSNR}_{\mathrm{app}}\uparrow$ & $F_{\mathrm{pos}}\uparrow$ & $F_{\mathrm{scl}}\uparrow$ & $F_{\mathrm{app}}\uparrow$ \\
\midrule
Scale-steered & 44.210 & 39.914 & 40.176 & 35.728 & 93.29 & 92.72 & 91.09 \\
Canonical-canvas & 43.501 & 39.568 & 39.945 & 36.148 & 93.00 & 93.94 & 89.17 \\
\bottomrule
\end{tabularx}
\par\medskip
\begin{tabularx}{\textwidth}{l>{\hsize=.85\hsize\linewidth=\hsize\centering\arraybackslash}X>{\hsize=1.2\hsize\linewidth=\hsize\centering\arraybackslash}X>{\hsize=1.2\hsize\linewidth=\hsize\centering\arraybackslash}X>{\hsize=1.2\hsize\linewidth=\hsize\centering\arraybackslash}X>{\hsize=.85\hsize\linewidth=\hsize\centering\arraybackslash}X>{\hsize=.85\hsize\linewidth=\hsize\centering\arraybackslash}X>{\hsize=.85\hsize\linewidth=\hsize\centering\arraybackslash}X}
\toprule
\rowcolor{tableheadgray}
Decoder & $E_{p\rightarrow c^\star}\downarrow$ & $E_{\hat c\rightarrow c^\star}\downarrow$ & $O_{\mathrm{attn}}\downarrow$ & $E_{\Delta p}\downarrow$ & $\beta_r\to1$ & $\beta_A\to2$ & $D_{\mathrm{app}}\downarrow$ \\
\midrule
Scale-steered & 0.0436 & 0.0391 & 0.0007 & 0.0051 & 1.0339 & 2.0225 & 0.0124 \\
Canonical-canvas & 0.0451 & 0.0428 & 0.0007 & 0.0046 & 0.9763 & 1.9315 & 0.0097 \\
\bottomrule
\end{tabularx}
\end{table}

The canonical-canvas implementation achieves higher transplant oracle PSNR
and $F_{\mathrm{scl}}$, with lower translation error and transplant drift.
The scale-steered implementation achieves higher reconstruction PSNR,
position and scale oracle PSNR, $F_{\mathrm{pos}}$, and $F_{\mathrm{app}}$.
We use the scale-steered implementation
in the main experiments to balance reconstruction quality and geometric control.

\subsection{Appearance-Transplant Pairing}
\label{app:appearance-pairing}

Each recipient scene replaces only one foreground slot, avoiding simultaneous changes to several footprints and their background competition. Pair eligibility is computed entirely from detached factual decompositions. We exclude inactive or degenerate supports, strong boundary truncation, and excessive foreground overlap. For MOVi-C, a detached factual-coherence refinement additionally requires OneFG compactness at least 1, at least 8 effective attention pixels, and IoU at least 0.6 between hard slot ownership and its gamma-2 OneFG support thresholded at 0.5. Obj3D uses no additional coherence refinement. The pair confidence $\omega_{ij}$ combines the base factual quality terms; it never observes the counterfactual outcome, so a poor transplant cannot escape supervision by invalidating its own pair.

This self-supervised factual filtering constructs training losses, where
privileged input-side object cores are unavailable.  Evaluation populations
and metric denominators are specified separately in
Sec.~\ref{app:obj3d-evaluation}.

The compactness target in the main paper implies the coverage relation
\begin{equation}
    A_{i\leftarrow j}^{\mathrm{target}}
    =\sg\!\left(K(m_j^{\don})\right)
     \sg\!\left(r(m_i^{\rec})\right)^2.
    \label{eq:implied-coverage}
\end{equation}
Thus recipient geometry commands spatial extent, while donor appearance may retain its fill density, transparency, and occlusion behavior.

\subsection{Single-Frame Schedule}
\label{app:single-frame-schedule}

The Obj3D curriculum uses 70k steps of single-frame training with random position initialization. Optimization uses AdamW with batch size 64, base learning rate $4\times10^{-4}$, zero weight decay, BF16 AMP, global gradient-norm clipping at $0.05$, a 10k-step linear warm-up followed by cosine decay to zero at 70k, and decoder compilation:
\begin{itemize}
    \item steps 0--15k: reconstruction warm-up, with the terminal-gate regularizer;
    \item steps 15k--40k: ramp and retain $\mathcal L_{\mathrm{pos}}$ and $\mathcal L_{\mathrm{ov}}$;
    \item steps 40k--50k: retain factual losses and add counterfactual appearance transplantation with $\mathcal L_{\mathrm{geo}}$, using the coordinate-wise center penalty;
    \item steps 50k--60k: retain these objectives and switch to the scale-normalized center penalty;
    \item steps 60k--70k: additionally enable $\mathcal L_{\mathrm{tail}}$ with $\lambda_{\mathrm{tail}}=10^{-3}$, without an additional ramp.
\end{itemize}
The center and tail penalties are defined in Sec.~\ref{app:center-tail-losses}.
The reported Obj3D configuration uses $\lambda_{\mathrm{pos}}=0.2$,
$\lambda_{\mathrm{ov}}=0.01$, $\lambda_{\mathrm{geo}}=0.002$,
$\lambda_K=0.25$, and $\lambda_{\mathrm{gate}}=10^{-4}$. Auxiliary losses
other than the late tail penalty enter through ramps near their activation boundaries. In the main trajectory,
the optional selector is jointly
trained from 35k with its disjoint AdamW optimizer~\citep{adamw} and is applied from 36.2k
with a ramped transfer cap. Its threshold and cap were selected on a
calibration panel drawn from the training split; no validation
or evaluation sample entered this panel. The selected operating point is then
used for subsequent training. The 40k diagnostic in
Fig.~\ref{fig:obj3d-appswap} comes from a separate controlled run using
reconstruction, $\mathcal L_{\mathrm{pos}}$, and $\mathcal L_{\mathrm{ov}}$
only. The diagnostic omits teacher--selector activation, counterfactual
geometry, and temporal initialization.
The selector's class/ranking weights are $\lambda_-=2$ and
$\lambda_{\mathrm{rank}}=0.25$.

\subsection{Temporal Initializer Semantics}
\label{app:temporal-initializer}

The explicit update is
\begin{equation}
\begin{aligned}
    (\delta\mathbf a,\delta\mathbf p,\delta\ell)_{t,i}
      &=\mathbf W_{\mathrm{init}}
        \bigl(\mathcal F_\Theta(\mathbf Z_{<t})\bigr)_i,\\
    \mathbf a_{t,i}^{(0)}
      &=\mathbf a_{t-1,i}+\delta\mathbf a_{t,i},\\
    \mathbf p_{t,i}^{(0)}
      &=\operatorname{clip}
        (\mathbf p_{t-1,i}+\delta\mathbf p_{t,i}),\\
    s_{t,i}^{(0)}
      &=\operatorname{clip}
        (s_{t-1,i}\exp(\delta\ell_{t,i})).
\end{aligned}
\label{eq:temporal-initializer}
\end{equation}
\begingroup
\emergencystretch=2em
The temporal module consumes the raw slot state $\mathbf z=[\mathbf a;\mathbf p;s]$ and predicts an initialization residual, not a separately evaluated future state. Its zero-initialized final projection returns $(\delta\mathbf a,\delta\mathbf p,\delta\ell)$: appearance uses an additive update, position a bounded additive update, and scale $s^{(0)}=\operatorname{clip}(s\exp(\delta\ell))$. Thus the update is additive in log-scale even though the Transformer input is raw scale. The residual adjusts appearance and geometry for next-frame ISA refinement. The time--slot Transformer~\citep{statm,attention} preserves the slot axis, mixes scene context through spatial attention, and returns one residual per slot. Ordered carry-forward provides slot-wise temporal correspondence for current-frame ISA refinement. Next-frame reconstruction jointly trains the Transformer and final projection through current-frame ISA refinement. Freezing or jointly adapting the calibrated single-frame core is a benchmark protocol choice.

\subsection{Gradient Routing Summary}
\label{app:gradient-routing}

Gradient routing is defined as follows:
\begin{itemize}
    \item $\mathcal L_{\mathrm{rec}}$ updates the full core.
    \item $\mathcal L_{\mathrm{ov}}$ keeps attention live.
    \item $\mathcal L_{\mathrm{pos}}$ keeps the OneFG target and selected-background logit paths live but detaches the position and scale commands, background-slot identities, and factual validity mask.
    \item $\mathcal L_{\mathrm{geo}}$ keeps the counterfactual alpha path, donor appearance, and recipient factual background logits live while detaching recipient commands, factual moments, background-slot identities, pair selection, and confidence.
    \item $\mathcal L_{\mathrm{tail}}$ detaches its threshold mask, foreground/background identities, and pair-confidence weights, while keeping target and background logits live. Its factual term follows the detached-command position-loss route; its counterfactual term uses donor appearance with detached recipient commands and the recipient's live factual background logits.
    \item teacher construction, selector inputs, and deployment eligibility are detached; $\mathcal L_{\mathrm{sel}}$ updates only selector parameters through a disjoint optimizer.
\end{itemize}
\endgroup

\section{Shared Evaluation Metrics and Notation}
\label{app:metric-definitions}

This section collects the metric notation shared by Obj3D and MOVi-C.
The populations, support roles, boundary scopes, and aggregation
rules are stated in the corresponding evaluation sections.

\paragraph{Decoded-support geometry.}
For a hard decoded support $\widehat M_i$ on the endpoint-aligned grid
$\mathbf q(\mathbf u)\in[-1,1]^2$, let
\begin{equation}
\begin{gathered}
 n_i=\sum_{\mathbf u}\widehat M_i(\mathbf u),\qquad
 \widehat{\mathbf c}_i=
 \frac{\sum_{\mathbf u}\widehat M_i(\mathbf u)\mathbf q(\mathbf u)}{n_i},\\
 \widehat r_i^2=
 \frac{\sum_{\mathbf u}\widehat M_i(\mathbf u)
 \|\mathbf q(\mathbf u)-\widehat{\mathbf c}_i\|_2^2}{n_i},\qquad
 \widehat A_i=\frac{n_i}{HW}.
\end{gathered}
\label{eq:shared-support-moments}
\end{equation}
Thus $\widehat{\mathbf c}_i$, $\widehat r_i$, and $\widehat A_i$ are the
decoded centroid, normalized-grid RMS radius, and coverage fraction.
Centroids are defined for supports with positive mass; radius and coverage are zero for zero-mass supports.

\paragraph{Factual geometry and allocation.}
Let $\mathbf p_i$ be the explicit position and $\mathbf c_i^\star$ an external
reference center.  The three factual position errors are
\begin{equation}
\begin{aligned}
 E_{p\rightarrow\hat c}&=\|\mathbf p_i-\widehat{\mathbf c}_i\|_2, &
 E_{p\rightarrow c^\star}&=\|\mathbf p_i-\mathbf c_i^\star\|_2,\\
 E_{\hat c\rightarrow c^\star}
 &=\|\widehat{\mathbf c}_i-\mathbf c_i^\star\|_2 .
\end{aligned}
\label{eq:shared-position-errors}
\end{equation}
All three use L2 distance on the normalized $[-1,1]^2$ grid.  For the set
$\mathcal I_x$ of attention maps participating in an example or frame $x$,
write $N_x=|\mathcal I_x|$ and spatially normalize each map so that
$\sum_n w_i(\mathbf u_n)=1$.  Normalized attention overlap is
\begin{equation}
 O_{\mathrm{attn}}(x)=
 \frac{1}{N_x(N_x-1)}
 \sum_{\substack{i,j\in\mathcal I_x\\i\ne j}}
 \sum_n w_i(\mathbf u_n)w_j(\mathbf u_n).
\label{eq:shared-attention-overlap}
\end{equation}
The reported $O_{\mathrm{attn}}$ averages this ordered off-diagonal overlap
over the stated evaluation population.

At fixed geometry $g=(\mathbf p,s)$, pooling all eligible decoded
appearances gives, for $x\in\{\widehat r,\widehat A\}$,
\begin{equation}
 \sigma_{x\mid g}=
 \sqrt{J^{-1}\sum_{j=1}^{J}(x_j-\bar x)^2}.
\label{eq:shared-fixed-geometry-std}
\end{equation}
Here $J$ is the number of defined measurements in the full population.
Each object--frame record has equal weight; $\bar x$ is its population mean.

\paragraph{Editing command adherence.}
For an edited mask $M_e$, let $M_t$ be the corresponding factual decoded
mask transformed by the editing command. Editing F1 is
\begin{equation}
 F_1(M_e,M_t)=\frac{2|M_e\cap M_t|}{|M_e|+|M_t|}.
\label{eq:shared-editing-f1}
\end{equation}
$F_{\mathrm{pos}}$, $F_{\mathrm{scl}}$, and $F_{\mathrm{app}}$ denote the
corresponding averages for position, scale, and appearance transplantation.
F1 is evaluated on valid edits (Sec.~\ref{app:obj3d-evaluation});
an empty edited support receives zero.
For a commanded pixel translation $\mathbf d_{\mathrm{px}}$, translation error is
\begin{equation}
 E_{\Delta p}=
 \frac{\|(
 \mathbf c_{\mathrm{px}}^e-\mathbf c_{\mathrm{px}}^f)
 -\mathbf d_{\mathrm{px}}\|_2}{\sqrt{H^2+W^2}},
\label{eq:shared-translation-error}
\end{equation}
where superscripts $f$ and $e$ denote factual and edited outputs.
Appearance-transplant drift is
\begin{equation}
 D_{\mathrm{app}}=
 \|\widehat{\mathbf c}_r^e-\widehat{\mathbf c}_r^f\|_2,
\label{eq:shared-appearance-drift}
\end{equation}
measured on the normalized grid while recipient geometry is fixed.

For a five-command scale curve, ordinary least squares with an intercept fits
$\log\widehat r(k)$ or $\log\widehat A(k)$ against $\log k$.  The resulting
slopes $\beta_r$ and $\beta_A$ have ideal values one and two.

For a video, the tube form pools intersection and support area over its $T$
frames before taking the ratio:
\begin{equation}
 F_{\mathrm{tube}}=
 \frac{2\sum_{t=1}^{T}|M_t^e\cap M_t^{\mathrm{ideal}}|}
 {\sum_{t=1}^{T}(|M_t^e|+|M_t^{\mathrm{ideal}}|)}.
\label{eq:shared-tube-f1}
\end{equation}

\endgroup
\section{Obj3D Evaluation Protocol and Extended Results}
\label{app:obj3d-evaluation}

\subsection{Models, Validation Population, and Support Roles}
The ISA baseline uses the scalar isotropic scale in Eq.~\eqref{eq:isa-moments}.
ISA, SlotAug, and GeoCo-SAVi use six slots with 64-dimensional appearance
states and matched encoder/decoder capacity, training exposure, and RGB
reconstruction resolution. SlotAug retains its published objectives and
commands; command normalization is applied in a subsequent hidden layer
to preserve instruction information. Reconstruction is evaluated with four
inference initializations. The geometry experiments below use the first frame
of all 2,600 validation windows, batch size four, and initialization seed
$42+b$ for batch index $b$. The validation-loader seed is 20260717.

\emph{Training support} is obtained from the model's decoder and attention
outputs. \emph{Measurement support} is the decoded mask,
\[
 \widehat M_i(\mathbf u)=\mathbb I[i=\arg\max_j\alpha_j(\mathbf u)].
\]
\emph{Reference support} denotes an RGB pseudo-label or a visible GT mask
used for the separate external position evaluation. Measured centroid,
radius, and coverage follow
Eq.~\ref{eq:shared-support-moments} on the endpoint-aligned grid.
The background slot maximizes factual mean alpha. We consider objects with
at least ten native-canvas support pixels to be valid, excluding negligible
supports. Valid edits additionally require the ideal transformed support
to meet the same criterion. Evaluation retains valid non-background objects
on the native $64\times64$ canvas. Background identity and eligibility
are fixed before intervention.

We report two scopes: \emph{all valid objects}, and
\emph{ideal-boundary-feasible objects}. The latter includes valid edits
whose factual and ideal transformed supports both avoid the outermost
pixel band. Feasibility is evaluated on transformed occupied pixel-cell bounds
before canvas clipping. For transplantation, both donor and recipient
factual supports and the transformed donor must satisfy this condition.
A scale curve is boundary-feasible only when all five commands are feasible.
Actual edited masks do not determine eligibility.

\subsection{Editing Commands and Aggregation}
\label{app:edit-mask-f1}
All interventions decode frozen factual slots. Translation uses four
pixel commands $(\pm6,0)$ and $(0,\pm6)$, converted to grid displacements
$\boldsymbol\delta=2\mathbf d_{\rm px}/63$. Scaling uses
$k\in\{0.5,0.75,1,1.25,1.5\}$. The ideal coordinate maps are
\[
 T_{\rm pos}(\mathbf q)=\mathbf q+\boldsymbol\delta,\qquad
 T_{\rm scl}(\mathbf q)=\mathbf p_i+k(\mathbf q-\mathbf p_i).
\]
ISA and GeoCo-SAVi modify the explicit position or scale; SlotAug uses its
native command-and-identity decoding path. Its scale translation component
compensates for the difference between its attention centroid and the
explicit position, so the ideal scale target is centered at $\mathbf p_i$.

For appearance transplantation, the donor is the next window from a
different source in cyclic validation order, with the same foreground-slot
rank. Both donor and recipient must be valid objects. The donor's
64 appearance coordinates replace the recipient's appearance while
recipient $(\mathbf p_r,s_r)$ and other slots are fixed. This is a
representation-level transplantation diagnostic for SlotAug as well.
The ideal donor-mask map is
\[
 T_{\rm app}(\mathbf q)=\mathbf p_r+\frac{s_r}{s_d}(\mathbf q-\mathbf p_d).
\]
Ideal binary masks are inverse-warped with nearest-neighbor sampling,
endpoint alignment, and zero padding. Donor cropping and occlusion remain
in its factual support.

Editing F-scores follow Eq.~\ref{eq:shared-editing-f1}. Scale F1 averages
the four nonidentity commands. The no-op
baseline scores the unchanged recipient mask against the target of the requested edit.
Main-text scores average valid edits without boundary filtering;
they measure adherence to the prescribed editing commands.

Position errors, appearance drift, and translation error follow
Eqs.~\ref{eq:shared-position-errors}--\ref{eq:shared-translation-error}.
On a square endpoint-aligned canvas, a diagonal-normalized distance converts
to grid L2 by multiplication by
\[
2\sqrt{H^2+W^2}/(W-1).
\]
PSNR uses dB, radius uses normalized-grid RMS-radius units, and coverage is
an area fraction.

Scale-response slopes follow
Sec.~\ref{app:metric-definitions}. Slopes require positive values at every
command. Center errors use outputs with defined centroids; their counts
are reported alongside F1 totals. Radius and coverage retain zero values
for zero-area outputs.

\subsection{Fixed-Geometry Stability}
Each valid non-background factual slot is decoded with its
factual background slot and attention held fixed.
The target is set to $\mathbf p=0$ and $s\in\{0.1,0.2,0.3\}$.
The pooled radius and coverage standard deviations follow
Eq.~\ref{eq:shared-fixed-geometry-std}.
Radius and coverage include zero-area outputs with value zero. No edited-support size or boundary
test is applied. Attention overlap is measured separately over all factual
normalized attention maps using Eq.~\ref{eq:shared-attention-overlap}.

\paragraph{Shared-background evaluation.}
We also evaluate cross-appearance stability with a shared factual
background. For each method, we sample 1,024 valid factual objects and
eight background slots from distinct videos using a fixed random seed;
the background source videos are shared across methods within each
dataset. Background slots follow the same factual mean-alpha rule above.
Each target is decoded with one background slot, whose representation
and attention are held fixed across the sampled appearances. Target
attention is retained, with $\mathbf p=0$ and
$s\in\{0.1,0.2,0.3\}$. All outputs contribute to the statistics, with
zero radius and coverage for zero-area supports.

For each background, we compute the mean, standard deviation, and
coefficient of variation (CV) across appearances, then average these
statistics over backgrounds. Table~\ref{tab:shared-background-stability}
reports $s=0.2$. GeoCo-SAVi has lower average within-background standard
deviations and CVs for radius and coverage than the corresponding
baselines at all three tested scales.

\begin{table}[htbp]
\centering
\small
\papertablecaption{Fixed-geometry stability with shared factual backgrounds at
$s=0.2$. Each statistic is computed across 1,024 appearances within a
background and averaged over eight backgrounds.}
\label{tab:shared-background-stability}
\setlength{\tabcolsep}{4pt}
\renewcommand{\arraystretch}{1.08}
\begin{tabular}{llrrrrrr}
\toprule
\rowcolor{tableheadgray}
& & \multicolumn{3}{c}{Radius} & \multicolumn{3}{c}{Coverage} \\
Dataset & Method & Mean & Std.$\downarrow$ & CV$\downarrow$ & Mean & Std.$\downarrow$ & CV$\downarrow$ \\
\midrule
Obj3D & ISA & 0.15430 & 0.04294 & 0.27831 & 0.03656 & 0.01793 & 0.49034 \\
& SlotAug & 0.14300 & 0.04978 & 0.34815 & 0.03266 & 0.02145 & 0.65703 \\
& GeoCo-SAVi & 0.15521 & \textbf{0.02029} & \textbf{0.13075} & 0.03364 & \textbf{0.00836} & \textbf{0.24846} \\
\midrule
MOVi-C & STAITUS & 0.25065 & 0.09912 & 0.49495 & 0.10652 & 0.07563 & 0.88901 \\
& GeoCo-SAVi & 0.14799 & \textbf{0.04616} & \textbf{0.32078} & 0.02830 & \textbf{0.01491} & \textbf{0.54265} \\
\bottomrule
\end{tabular}
\end{table}

\subsection{External RGB Reference Evaluation}
Obj3D supplies no official instance masks. A model-independent RGB bank
provides instance masks.
Color and temporal cues identify object regions; RGB-based refinement
separates touching instances and joins fragmented supports. A one-pixel
boundary allowance covers object-edge pixels.

The bank provides 9,265 objects across all 2,600 validation windows.
Hungarian maximum-overlap association maps the reference objects to factual
decoded slots. The associated slot's position
and full decoded centroid are compared with the reference-mask centroid.
The shared evaluation set contains 9,255 objects with defined centroids
across the compared methods and controls. Table~\ref{tab:external-centers}
gives the three distances on this population; the all-slot position
measurement is reported separately.

\begin{table}[htbp]
\centering
\small
\papertablecaption{External position evaluation in normalized-grid L2 units: Obj3D uses 9,255 RGB pseudo-reference objects from all 2,600 windows, and MOVi-C uses 1,543 visible GT tracks. Within-model distance is evaluated on the same associated objects.}
\label{tab:external-centers}
\setlength{\tabcolsep}{4pt}
\renewcommand{\arraystretch}{1.18}
\begin{tabular}{lrrr}
\toprule
\rowcolor{tableheadgray}
Dataset / method & $E_{p\rightarrow\hat c}$ & $E_{p\rightarrow c^\star}$ & $E_{\hat c\rightarrow c^\star}$ \\
\midrule
Obj3D / ISA & 0.10135 & 0.10232 & 0.04204 \\
Obj3D / SlotAug & 0.04658 & 0.05794 & 0.03584 \\
Obj3D / GeoCo-SAVi & 0.01126 & 0.04361 & 0.03910 \\
MOVi-C / STAITUS & 0.08683 & 0.34170 & 0.31277 \\
MOVi-C / GeoCo-SAVi & 0.02761 & 0.24032 & 0.23406 \\
\bottomrule
\end{tabular}
\end{table}
\begin{table}[htbp]
\centering
\small
\papertablecaption{All-slot position error in normalized-grid L2 units. Position uses original factual ownership; video editing uses a consistent decoder precision for factual reference and intervention, so valid-object counts can differ slightly.}
\label{tab:allslot-position}
\setlength{\tabcolsep}{4pt}
\renewcommand{\arraystretch}{1.18}
\begin{tabular}{lrrr}
\toprule
\rowcolor{tableheadgray}
Dataset / method & $E_{p\rightarrow\hat c}$ & Defined & Edit objects \\
\midrule
Obj3D / ISA & 0.10489 & 9424 & 9424 \\
Obj3D / SlotAug & 0.04763 & 9357 & 9357 \\
Obj3D / GeoCo-SAVi & 0.01065 & 9320 & 9320 \\
MOVi-C / STAITUS & 0.06865 & 51962 & 51971 \\
MOVi-C / GeoCo-SAVi & 0.02567 & 51600 & 51629 \\
\bottomrule
\end{tabular}
\end{table}
\subsection{Extended Numerical Results}
\label{app:fscore-intervals}
\begingroup
\setlength{\intextsep}{10pt plus 2pt minus 2pt}
\setlength{\textfloatsep}{14pt plus 2pt minus 2pt}
\setlength{\floatsep}{8pt plus 2pt minus 2pt}
Tables~\ref{tab:edit-mask-f1}--\ref{tab:obj3d-new-noop} report editing,
scale response, sample counts, fixed-geometry variation, and no-op scores.
Both scopes retain all eligible actual outcomes. F1 additionally reports
valid edits in the interior subset.
Section~\ref{app:factual-scale-calibration} reports factual scale calibration on both datasets.

F-score confidence intervals use 10,000 source-video bootstrap replicates
and the 2.5th and 97.5th percentiles, preserving each reported score's
aggregation. All observations from a video share its resampling weight;
transplant pairs receive the product of recipient and donor video weights
from the same resample. The intervals summarize evaluation-sample variability.

\begin{table}[!htb]
\centering
\small
\papertablecaption{Obj3D editing results. F1 (\%) uses valid edits in each scope; interior uses factual/ideal feasibility. Brackets give 95\% confidence intervals for the main-table F-scores. $E_{\Delta p}$ is a diagonal fraction and $D_{\rm app}$ is grid L2.}
\label{tab:edit-mask-f1}
\setlength{\tabcolsep}{4pt}
\renewcommand{\arraystretch}{1.18}
\begin{tabular}{lrrrrr}
\toprule
\rowcolor{tableheadgray}
Method / scope & $F_{\rm pos}$ & $F_{\rm scl}$ & $F_{\rm app}$ & $E_{\Delta p}$ & $D_{\rm app}$ \\
\midrule
ISA (all) & 86.35 & 62.84 & 75.73 & 0.01012 & 0.05410 \\
& [85.84, 86.83] & [62.40, 63.28] & [74.44, 76.93] & & \\
SlotAug (all) & 90.78 & 82.19 & 78.20 & 0.00770 & 0.06012 \\
& [90.34, 91.22] & [82.03, 82.34] & [76.79, 79.49] & & \\
GeoCo-SAVi (all) & 93.29 & 92.72 & 91.09 & 0.00645 & 0.02083 \\
& [93.03, 93.55] & [92.60, 92.83] & [90.43, 91.64] & & \\
ISA (interior) & 85.83 & 61.08 & 78.79 & 0.00908 & 0.04422 \\
SlotAug (interior) & 91.19 & 81.75 & 84.65 & 0.00599 & 0.03214 \\
GeoCo-SAVi (interior) & 93.40 & 93.20 & 93.08 & 0.00509 & 0.01240 \\
\bottomrule
\end{tabular}
\end{table}
\begin{table}[!htb]
\centering
\small
\papertablecaption{Obj3D five-point scale response. Slopes are means over defined curves; counts are defined/total.}
\label{tab:obj3d-new-curves}
\setlength{\tabcolsep}{4pt}
\renewcommand{\arraystretch}{1.18}
\begin{tabular}{lrrrr}
\toprule
\rowcolor{tableheadgray}
Method / scope & $\beta_r$ & $\beta_A$ & $\beta_r$ def. & $\beta_A$ def. \\
\midrule
ISA (all) & 0.09882 & -0.01512 & 9344/9424 & 9370/9424 \\
SlotAug (all) & 0.54612 & 1.04735 & 9339/9357 & 9344/9357 \\
GeoCo-SAVi (all) & 1.03473 & 1.98611 & 9303/9320 & 9310/9320 \\
ISA (interior) & 0.09388 & -0.00923 & 6928/6952 & 6939/6952 \\
SlotAug (interior) & 0.54118 & 1.05748 & 7310/7314 & 7311/7314 \\
GeoCo-SAVi (interior) & 1.03391 & 2.02246 & 7380/7386 & 7384/7386 \\
\bottomrule
\end{tabular}
\end{table}
\begin{table}[!htb]
\centering
\small
\papertablecaption{Obj3D scope-specific F1 denominators and defined/total center-error counts. F1 retains all actual outputs for valid edits, including empty outputs.}
\label{tab:obj3d-new-counts}
\setlength{\tabcolsep}{4pt}
\renewcommand{\arraystretch}{1.18}
\begin{tabular}{lrrrrr}
\toprule
\rowcolor{tableheadgray}
Method / scope & $N_{\rm pos}$ & $N_{\rm scl}$ & $N_{\rm app}$ & $E_{\Delta p}$ def. & $D_{\rm app}$ def. \\
\midrule
ISA (all) & 37144 & 36631 & 7402 & 37511/37696 & 7417/7418 \\
SlotAug (all) & 36843 & 36619 & 7691 & 37068/37428 & 7689/7691 \\
GeoCo-SAVi (all) & 36616 & 36858 & 7397 & 37221/37280 & 7397/7397 \\
ISA (interior) & 30842 & 30357 & 4564 & 30831/30842 & 4564/4564 \\
SlotAug (interior) & 31390 & 31258 & 5293 & 31328/31390 & 5291/5293 \\
GeoCo-SAVi (interior) & 30319 & 30654 & 4698 & 30319/30319 & 4698/4698 \\
\bottomrule
\end{tabular}
\end{table}
\begin{table}[!htb]
\centering
\small
\papertablecaption{Obj3D conditional output diagnostics, shown as count/eligible count (\%). Disappearance counts empty outputs for positive-area ideal targets.}
\label{tab:obj3d-new-empty}
\setlength{\tabcolsep}{4pt}
\renewcommand{\arraystretch}{1.18}
\begin{tabular}{lrrr}
\toprule
\rowcolor{tableheadgray}
Method / scope & Position & Scale & Transplant \\
\midrule
ISA (within-canvas) & 11/30842 (0.04) & 69/30357 (0.23) & 0/4564 (0.00) \\
SlotAug (within-canvas) & 62/31390 (0.20) & 8/31258 (0.03) & 2/5293 (0.04) \\
GeoCo-SAVi (within-canvas) & 0/30319 (0.00) & 4/30654 (0.01) & 0/4698 (0.00) \\
\midrule
ISA (all) & 79/37358 (0.21) & 70/37696 (0.19) & 1/7418 (0.01) \\
SlotAug (all) & 78/37064 (0.21) & 16/37428 (0.04) & 2/7691 (0.03) \\
GeoCo-SAVi (all) & 1/36805 (0.00) & 12/37280 (0.03) & 0/7397 (0.00) \\
\bottomrule
\end{tabular}
\end{table}
\begin{table}[!htb]
\centering
\small
\papertablecaption{Obj3D pooled cross-appearance fixed-geometry statistics. The radius coefficient of variation is $\mathrm{CV}_{r\mid g}=\sigma_{r\mid g}/\bar r$. At $s=0.2$, GeoCo-SAVi and ISA have similar mean radii, with lower radius standard deviation and CV for GeoCo-SAVi. $N$ counts all measurements.}
\label{tab:obj3d-new-fixed}
\setlength{\tabcolsep}{4pt}
\renewcommand{\arraystretch}{1.18}
\begin{tabular}{lrrrrrr}
\toprule
\rowcolor{tableheadgray}
Method & $s$ & $\bar r$ & $\sigma_{r\mid g}$ & $\mathrm{CV}_{r\mid g}$ & $\sigma_{A\mid g}$ & $N$ \\
\midrule
ISA & 0.1 & 0.14826 & 0.03351 & 0.22600 & 0.01239 & 9424 \\
ISA & 0.2 & 0.15308 & 0.04405 & 0.28775 & 0.01833 & 9424 \\
ISA & 0.3 & 0.15779 & 0.07107 & 0.45039 & 0.02911 & 9424 \\
SlotAug & 0.1 & 0.12094 & 0.02476 & 0.20475 & 0.00874 & 9357 \\
SlotAug & 0.2 & 0.14279 & 0.05201 & 0.36426 & 0.02265 & 9357 \\
SlotAug & 0.3 & 0.10030 & 0.10388 & 1.03572 & 0.04049 & 9357 \\
GeoCo-SAVi & 0.1 & 0.07268 & 0.01056 & 0.14530 & 0.00208 & 9320 \\
GeoCo-SAVi & 0.2 & 0.15386 & 0.02109 & 0.13709 & 0.00856 & 9320 \\
GeoCo-SAVi & 0.3 & 0.23137 & 0.03109 & 0.13439 & 0.01907 & 9320 \\
\bottomrule
\end{tabular}
\end{table}
\begin{table}[!htb]
\centering
\small
\papertablecaption{Obj3D no-op F1 (\%): unchanged recipient masks scored against the requested edit targets, on the same valid edits without boundary filtering.}
\label{tab:obj3d-new-noop}
\setlength{\tabcolsep}{4pt}
\renewcommand{\arraystretch}{1.18}
\begin{tabular}{lrrr}
\toprule
\rowcolor{tableheadgray}
Method & $F_{\rm pos}^{\rm noop}$ & $F_{\rm scl}^{\rm noop}$ & $F_{\rm app}^{\rm noop}$ \\
\midrule
ISA & 34.62 & 62.54 & 69.51 \\
SlotAug & 33.56 & 63.54 & 77.83 \\
GeoCo-SAVi & 37.02 & 63.54 & 80.04 \\
\bottomrule
\end{tabular}
\end{table}
\subsection{Object-Matched Editing and RGB Preservation}
\label{app:obj3d-edit-validation}
We use the RGB instance bank above to compare edits of the same reference
objects across ISA, SlotAug, and GeoCo-SAVi. Factual maximum-overlap
association fixes the slot for each reference object. The commands are
the four translations and four nonidentity scales in
Sec.~\ref{app:edit-mask-f1}. All results in this subsection average edits
within each source video and then equally across the 200 videos.

\paragraph{Common-object editing.}
The common pool requires valid reference edits and valid edits for all
three associated decoded supports. It is fixed from factual supports and
the prescribed transformations. Editing F1 follows
Eq.~\ref{eq:shared-editing-f1} and the coordinate maps in
Sec.~\ref{app:edit-mask-f1}. Table~\ref{tab:obj3d-common-edit}
reports 36,210 translation and 35,986 scale edits, each spanning
9,216 reference objects and all 2,600 validation windows.

\begin{table}[htbp]
\centering
\small
\papertablecaption{Obj3D command adherence on common reference object--command pairs. F1 is in percent; each column uses the same pairs for all methods.}
\label{tab:obj3d-common-edit}
\setlength{\tabcolsep}{14pt}
\renewcommand{\arraystretch}{1.18}
\begin{tabular}{lrr}
\toprule
\rowcolor{tableheadgray}
Method & $F_{\rm pos}\uparrow$ & $F_{\rm scl}\uparrow$ \\
\midrule
ISA & 88.65 & 63.85 \\
SlotAug & 91.98 & 82.62 \\
GeoCo-SAVi & \textbf{93.77} & \textbf{92.93} \\
\bottomrule
\end{tabular}
\end{table}

\paragraph{Independent RGB targets.}
An RGB visible-layer oracle provides a common target image for each
reference object and command. We extract the object using its RGB-derived
matte, fill its source region by local background inpainting, and composite
the transformed premultiplied RGB layer with bilinear sampling. Scaling
is centered at the reference-mask centroid. These target images are
constructed from the input RGB and reference masks independently of the
model predictions. We evaluate all eight commands for each of 9,240 valid
reference objects across the 2,600 windows, giving 73,920 edits per method.
Table~\ref{tab:obj3d-oracle-rgb} reports PSNR on the full $64\times64$ image
and on the union of the source and transformed reference supports,
expanded by two four-connected pixel steps.

For transplantation, each valid reference object is paired with a donor
from the next video in cyclic order, giving 9,240 common pairs across
200 recipient videos. The visible donor layer is aligned to the recipient
reference-mask centroid and RMS radius, then composited over the inpainted
recipient image. The associated donor appearance state replaces the
recipient appearance state, with geometry and other slots fixed.
All pairs are retained without boundary filtering. The editing region
includes the recipient and transformed donor supports and alpha footprints,
expanded as above. Scores are averaged within recipient videos and then
equally over videos.

\begin{table}[htbp]
\centering
\small
\papertablecaption{Obj3D PSNR against independent RGB oracle edits (dB). Position and scale each contain 36,960 edits; transplant contains 9,240 pairs. Target images are identical across methods.}
\label{tab:obj3d-oracle-rgb}
\setlength{\tabcolsep}{6pt}
\renewcommand{\arraystretch}{1.18}
\begin{tabular}{lrrrrrr}
\toprule
& \multicolumn{3}{c}{Full image} & \multicolumn{3}{c}{Editing region} \\
\cmidrule(lr){2-4}\cmidrule(lr){5-7}
\rowcolor{tableheadgray}
Method & Position$\uparrow$ & Scale$\uparrow$ & Transplant$\uparrow$ & Position$\uparrow$ & Scale$\uparrow$ & Transplant$\uparrow$ \\
\midrule
ISA & 37.322 & 35.368 & 34.606 & 27.675 & 24.349 & 23.327 \\
SlotAug & 38.654 & 38.528 & 34.951 & 29.720 & 28.924 & 24.047 \\
GeoCo-SAVi & \textbf{39.914} & \textbf{40.176} & \textbf{35.728} & \textbf{30.883} & \textbf{30.889} & \textbf{24.428} \\
\bottomrule
\end{tabular}
\end{table}

\paragraph{Non-target RGB preservation.}
On the position and scale object--command pairs, we measure RGB changes outside the
editing region defined above. This region is shared across methods.
For RGB intensities in $[0,1]$, we report the mean absolute change from
the factual reconstruction and the percentage of pixels whose
channel-averaged absolute change exceeds $2/255$.
GeoCo-SAVi has lower values for both measures
(Table~\ref{tab:obj3d-rgb-preservation}). Paired 95\% confidence intervals
from 5,000 source-video bootstrap replicates exclude zero for the
reported improvements over both baselines in
Tables~\ref{tab:obj3d-common-edit}--\ref{tab:obj3d-rgb-preservation}.

\begin{table}[htbp]
\centering
\small
\papertablecaption{Obj3D RGB changes outside the source/target editing region. MAE is multiplied by $10^3$; changed-pixel rates are percentages.}
\label{tab:obj3d-rgb-preservation}
\setlength{\tabcolsep}{8pt}
\renewcommand{\arraystretch}{1.18}
\begin{tabular}{lrrrr}
\toprule
& \multicolumn{2}{c}{RGB MAE $\times10^3$} & \multicolumn{2}{c}{Changed pixels (\%)} \\
\cmidrule(lr){2-3}\cmidrule(lr){4-5}
\rowcolor{tableheadgray}
Method & Position$\downarrow$ & Scale$\downarrow$ & Position$\downarrow$ & Scale$\downarrow$ \\
\midrule
ISA & 0.8902 & 0.7724 & 2.1480 & 2.0424 \\
SlotAug & 0.8296 & 0.7756 & 2.1646 & 1.9778 \\
GeoCo-SAVi & \textbf{0.4939} & \textbf{0.6873} & \textbf{1.2360} & \textbf{1.6966} \\
\bottomrule
\end{tabular}
\end{table}
\endgroup
\subsection{Component Comparisons}
\label{app:component-comparisons}
The seven controls in Table~\ref{tab:ablation} are evaluated at 70k updates.
The conventional decoder replaces scale-steered spatial convolutions with
ordinary unit-grid convolutions, retaining the relative-coordinate input,
feature widths, upsampling stages, and output heads.

The loss controls set the corresponding objective weight to zero:
$\mathcal L_{\rm ov}$, $\mathcal L_{\rm pos}$,
$\mathcal L_{\rm geo}$, or $\mathcal L_{\rm tail}$.
Removing $\mathcal L_{\rm geo}$ disables its center, radius, and compactness
terms; the counterfactual branch remains available for the tail
penalty. The readout control disables the terminal readout and its gate
regularizer. The selector control disables the alpha selector.
For active losses, factual calibration and overlap regularization begin
at 15k, and counterfactual calibration begins at 40k with the ramps
specified above. The counterfactual center penalty uses coordinate Huber
before 50k and the scale-normalized form thereafter. The tail penalty,
when active, has weight $0.001$ and threshold $0.25$ from 60k without a
ramp. These settings complement the component-specific comparisons in
the main text.

PSNR averages four inference
initializations. Editing uses the same 2,600 validation windows and
valid-object criterion as the main evaluation.
Reference-center errors use the fixed RGB reference bank described above;
fixed-geometry spreads pool objects at $p=0,s=0.2$.

\section{MOVi-C Protocol and Extended Diagnostics}
\label{app:movic}

\subsection{Common Video Protocol and Reference Implementations}

All three methods are trained on MOVi-C~\citep{kubric} for 100k optimizer
updates and evaluated on the fixed official validation population of 250
videos, each containing 24 frames.  RGB reconstruction uses $64\times64$
frames; instance metrics use the official GT resolution after alpha resizing,
and editing protocols state their measurement resolution below. Native RGB is
first resized bilinearly to $64\times64$, then bicubically to $336\times336$
for the frozen visual encoder
\texttt{vit\_small\_\allowbreak patch14\_\allowbreak dinov2.lvd142m}~\citep{dinov2}; its block-12 384-dimensional
tokens are projected to 128
dimensions.  The reference implementations are matched in data, RGB target,
encoder/decoder capacity, batch size, 100k update budget, and total sample
exposure.  Loss, optimizer, and curriculum hyperparameters
otherwise use the published values, subject only to the adaptations disclosed
below.

\subsection{GeoCo-SAVi MOVi-C Architecture, Optimization, and Curriculum}
\label{app:movic-curriculum}

\paragraph{Slot core and renderer.}
GeoCo-SAVi uses 11 slots with a 64-dimensional appearance state and explicit
state $(\mathbf a,\mathbf p,s)$.  It uses 64-dimensional query/key/value
projections, a 64-dimensional Slot-Attention GRU~\citep{slotattention,gru}, a 256-dimensional update
MLP, and a 128-dimensional grid encoding.  The first frame of every
training clip uses three Slot-Attention iterations; a
carry-forward frame uses two.  Positions are initialized iid-uniformly,
$s$ is clamped to $[.001,2]$, factor 5 normalizes the object-relative
coordinate field, and the
Spatially Equivariant Decoder has width 64.  It forms a $32\times32$ canonical
field, applies one scale-steered $5\times5$ refinement, bilinearly upsamples,
then applies four scale-steered $3\times3$ refinements at $64\times64$.
The RGB head is $3\times3$ with sigmoid output and the alpha head is $1\times1$. The reported renderer uses only the convolutional path described above. The regular steered kernel is $3\times3$ and its learned gain is
bounded by 2.

\paragraph{Optimization and reference-scale continuation.}
We use AdamW with base learning rate $4\times10^{-4}$, zero weight decay,
BF16 AMP, global gradient-norm clipping at $0.05$, physical/effective batch
size 16.  Only the decoder is compiled.  The core rate
warms up linearly for 10k updates and then follows cosine decay to zero at
100k.  The training loader has eight workers, no subsampling, and uses video
identity as its source key; the fixed position seed for evaluation is
20260723.  Let $s_c=0.851256531$, $s_f=0.2$, and
$\ell_t=(1-\beta_t)\log s_c+\beta_t\log s_f$.  The global numerical reference of the
scale-steered operator follows
\begin{equation}
s_{\mathrm{ref}}(t)=
\begin{cases}
 s_c, & t<2000,\\
 e^{\ell_t},
    & 2000\le t<10000,\\
 s_f, & t\ge10000,
\end{cases}
\label{eq:movic-sref}
\end{equation}
where $\beta_t=\min((t-1999)/8000,1)$.
This reference is a global numerical continuation for the convolution
operator, not an object-size prior or a replacement for the predicted $s_i$.

\paragraph{Loss activation and temporal curriculum.}
Position and normalized-attention overlap activate at 15k and reach full
weight over 2k updates.  We report the rounded coefficients
$\lambda_{\mathrm{pos}}=0.04$ and $\lambda_{\mathrm{ov}}=0.018$.
Position uses coordinate Huber loss
with threshold $0.05$ on the detached-command gamma-2 OneFG centroid; only
loose factual foreground supports with OneFG coverage at most 0.2 pass the
detached validity filter, and no separate valid-background gate is applied.
$\mathcal L_{\mathrm{geo}}$ is kept at zero until 48k, ramps linearly to full
weight by 50k, and then uses
$\lambda_{\mathrm{geo}}=0.004$, unit center/radius terms, and a 0.25
compactness term, all with Huber threshold $0.05$.  Its gamma-2 OneFG pairs
use factual quality gates and the receiver-center/radius, donor-compactness
semantics described in Sec.~\ref{sec:method}.  It is evaluated at time index
0 for two-frame clips and at indices 0 and 2 for four-frame clips.

\begin{table}[htbp]
\centering
\small
\papertablecaption{Reported GeoCo-SAVi MOVi-C curriculum.  All stages share one optimizer,
scheduler, and AMP-scaler trajectory.}
\label{tab:movic-curriculum}
\renewcommand{\arraystretch}{1.18}
\setlength{\tabcolsep}{4pt}
\begin{tabularx}{\textwidth}{cc>{\raggedright\arraybackslash}X>{\raggedright\arraybackslash}X}
\toprule
Updates & Frames / clip & Active terms & Purpose \\
\midrule
0--15k & 2 & RGB reconstruction; $s_{\mathrm{ref}}$ continuation
    & Stable slot formation with carry-forward \\
15--17k & 2 & $\mathcal L_{\mathrm{pos}}$, $\mathcal L_{\mathrm{ov}}$ ramp
    & Factual geometry and allocation calibration \\
17--48k & 2 & $\mathcal L_{\mathrm{pos}}$, $\mathcal L_{\mathrm{ov}}$
    & Core-only stabilization \\
48--50k & 2 & $\mathcal L_{\mathrm{geo}}$ ramp
    & Counterfactual geometry activation \\
50--70k & 2 & Reconstruction, factual and counterfactual objectives
    & Appearance-transplantation calibration \\
70--72k & 4 & Zero-residual STATM local LR warm-up
    & Temporal-module activation \\
72--100k & 4 & Full STATM local schedule
    & Joint temporal refinement \\
\bottomrule
\end{tabularx}
\end{table}

The residual STATM initializer has width 128, four heads, a 512-dimensional
MLP, and four history states.  It detaches history and uses zero dropout.  Its
output projection is zero-initialized, so the transition at 70k starts from
exact carry-forward.  Its parameter group belongs to the same AdamW instance
from update zero but receives no gradient before 70k; it then warms up locally
through 72k to $0.25$ times the base rate, then follows an independent cosine
decay to zero at 100k.
The reported MOVi-C configuration combines the spatially equivariant core
with the temporal initializer; the optional terminal appearance readout,
selective alpha ownership, and tail penalty are inactive.

The STAITUS implementation uses the same-capacity SAVi encoder/decoder core, a
temporal appearance predictor, and an active-set gate, with the published losses. The gate is a part of reconstruction
and the losses, with a 15k all-active warm-up followed
by a 15k gate ramp. STATM-SAVi uses
the same-capacity SAVi encoder/decoder core, RGB MSE, and the published
corresponding-slot spatial and temporal Transformer in its strongest CS/T+S
form.  These are same-protocol
implementations for controlled comparison.

\subsection{Factual Reconstruction and Instance Metrics}

Full-softmax alpha is bilinearly resized to the official GT
resolution before taking a slotwise argmax.  FG-ARI, global ARI, Hungarian
mIoU, and mean best overlap (mBO) are averaged with equal weight per video;
global ARI retains the GT background.  RGB values lie in $[0,1]$; PSNR is
$10\log_{10}(1/\mathrm{MSE})$ per 24-frame video, with MSE averaged over all
RGB entries in that video and the 250 video-level scores weighted equally.
Neither the factual metrics nor their denominators use
quality filters, OneFG, explicit geometry, or edited states.  Official visible
instance segmentations are used only for these metrics, never as model input.

For factual attention overlap, the final Slot Attention allocation
maps supply $w_i$ to Eq.~\ref{eq:shared-attention-overlap}, and the metric is
averaged over every frame of all 250 validation videos.
STAITUS is evaluated only on its factual active slots, consistent with the
state that contributes to its reconstruction.  This convention is favorable
to STAITUS because closed slots are excluded, whereas GeoCo-SAVi's all-slot
average also includes diffuse empty-slot maps.  This row is a factual
allocation diagnostic, not an alpha-overlap or editable-geometry metric.

\subsection{All-Slot Editing and External Position Evaluation}
Editing uses all 24 frames of all 250 validation videos and the
valid non-background-object rule in Sec.~\ref{app:obj3d-evaluation}.
Each factual slot is edited independently using the four translations and
five-point scale grid. Both factual edit references and interventions use
the same decoder precision. STAITUS preserves its active-slot gating.
The all-object and ideal-boundary-feasible scopes use the same definitions
as Obj3D, with no cross-method slot association.
For transplantation, donor windows are taken from a different video in
cyclic order, with the same foreground-slot rank. Editing F-scores quantify
command adherence using Eq.~\ref{eq:shared-editing-f1}.

External position errors use all 1,543 area-qualified GT tracks.
The reference frame maximizes visible GT area, with earliest-frame
tie-breaking. Tracks require at least 20 reference pixels; a one-pixel
eroded core, or the original mask if erosion empties it, is used for
factual Hungarian association. Each association is frozen before measurement.
The visible GT centroid supplies $\mathbf c^\star$ in
Table~\ref{tab:external-centers}. These errors describe both latent position
and decoded-mask localization relative to visible objects.
The per-frame editing experiment uses every eligible slot independently
of this external association.

\subsection{Extended Numerical Results}
Tables~\ref{tab:movic-new-edits}--\ref{tab:movic-new-noop} give the
per-frame all-slot measurements and their denominators.
Section~\ref{app:factual-scale-calibration} reports factual scale calibration on both datasets.
Boundary feasibility is computed from factual masks and their ideal transforms;
the resulting evaluation counts are reported for all methods.
\begingroup
\setlength{\intextsep}{10pt plus 2pt minus 2pt}
\setlength{\textfloatsep}{14pt plus 2pt minus 2pt}
\setlength{\floatsep}{8pt plus 2pt minus 2pt}

\begin{table}[!htb]
\centering
\small
\papertablecaption{MOVi-C editing results. F1 (\%) uses valid edits in each scope; interior uses factual/ideal feasibility. $E_{\Delta p}$ is a diagonal fraction and $D_{\rm app}$ is grid L2.}
\label{tab:movic-new-edits}
\setlength{\tabcolsep}{4pt}
\renewcommand{\arraystretch}{1.18}
\begin{tabular}{lrrrrr}
\toprule
\rowcolor{tableheadgray}
Method / scope & $F_{\rm pos}$ & $F_{\rm scl}$ & $F_{\rm app}$ & $E_{\Delta p}$ & $D_{\rm app}$ \\
\midrule
STAITUS (all) & 75.96 & 76.86 & 60.37 & 0.02442 & 0.08065 \\
GeoCo-SAVi (all) & 83.01 & 83.11 & 72.04 & 0.01045 & 0.04083 \\
STAITUS (interior) & 74.25 & 76.28 & 58.33 & 0.01745 & 0.05848 \\
GeoCo-SAVi (interior) & 87.86 & 85.85 & 76.02 & 0.00667 & 0.03335 \\
\bottomrule
\end{tabular}
\end{table}
\begin{table}[!htb]
\centering
\small
\papertablecaption{MOVi-C five-point scale response. Slopes are means over defined curves; counts are defined/total.}
\label{tab:movic-new-curves}
\setlength{\tabcolsep}{4pt}
\renewcommand{\arraystretch}{1.18}
\begin{tabular}{lrrrr}
\toprule
\rowcolor{tableheadgray}
Method / scope & $\beta_r$ & $\beta_A$ & $\beta_r$ def. & $\beta_A$ def. \\
\midrule
STAITUS (all) & 0.70923 & 1.12247 & 50738/51971 & 50964/51971 \\
GeoCo-SAVi (all) & 0.85099 & 1.37988 & 50631/51629 & 50916/51629 \\
STAITUS (interior) & 0.73223 & 1.21219 & 12042/12159 & 12056/12159 \\
GeoCo-SAVi (interior) & 0.88920 & 1.61780 & 16796/16866 & 16802/16866 \\
\bottomrule
\end{tabular}
\end{table}
\begin{table}[!htb]
\centering
\small
\papertablecaption{MOVi-C scope-specific F1 denominators and defined/total center-error counts. F1 retains all actual outputs for valid edits, including empty outputs.}
\label{tab:movic-new-counts}
\setlength{\tabcolsep}{4pt}
\renewcommand{\arraystretch}{1.18}
\begin{tabular}{lrrrrr}
\toprule
\rowcolor{tableheadgray}
Method / scope & $N_{\rm pos}$ & $N_{\rm scl}$ & $N_{\rm app}$ & $E_{\Delta p}$ def. & $D_{\rm app}$ def. \\
\midrule
STAITUS (all) & 204961 & 201269 & 44284 & 205891/207884 & 42904/45177 \\
GeoCo-SAVi (all) & 198500 & 189003 & 41954 & 198106/206516 & 43175/43950 \\
STAITUS (interior) & 72936 & 68173 & 7181 & 72030/72936 & 6764/7181 \\
GeoCo-SAVi (interior) & 113245 & 103197 & 12059 & 112427/113245 & 11839/12059 \\
\bottomrule
\end{tabular}
\end{table}
\begin{table}[!htb]
\centering
\small
\papertablecaption{MOVi-C conditional output diagnostics, shown as count/eligible count (\%). Disappearance counts empty outputs for positive-area ideal targets.}
\label{tab:movic-new-empty}
\setlength{\tabcolsep}{4pt}
\renewcommand{\arraystretch}{1.18}
\begin{tabular}{lrrr}
\toprule
\rowcolor{tableheadgray}
Method / scope & Position & Scale & Transplant \\
\midrule
STAITUS (within-canvas) & 906/72936 (1.24) & 792/68173 (1.16) & 417/7181 (5.81) \\
GeoCo-SAVi (within-canvas) & 818/113245 (0.72) & 559/103197 (0.54) & 220/12059 (1.82) \\
\midrule
STAITUS (all) & 1494/206317 (0.72) & 1227/207586 (0.59) & 2263/45098 (5.02) \\
GeoCo-SAVi (all) & 4172/201674 (2.07) & 901/206414 (0.44) & 760/43820 (1.73) \\
\bottomrule
\end{tabular}
\end{table}
\begin{table}[!htb]
\centering
\small
\papertablecaption{MOVi-C pooled cross-appearance fixed-geometry statistics. Radius CV is $\sigma_{r\mid g}/\bar r$. GeoCo-SAVi has lower radius standard deviation and CV at all three scales. $N$ counts all measurements.}
\label{tab:movic-new-fixed}
\setlength{\tabcolsep}{4pt}
\renewcommand{\arraystretch}{1.18}
\begin{tabular}{lrrrrrr}
\toprule
\rowcolor{tableheadgray}
Method & $s$ & $\bar r$ & $\sigma_{r\mid g}$ & $\mathrm{CV}_{r\mid g}$ & $\sigma_{A\mid g}$ & $N$ \\
\midrule
STAITUS & 0.1 & 0.17068 & 0.16820 & 0.98550 & 0.11655 & 51962 \\
STAITUS & 0.2 & 0.24798 & 0.17560 & 0.70813 & 0.14393 & 51962 \\
STAITUS & 0.3 & 0.31203 & 0.18789 & 0.60217 & 0.17095 & 51962 \\
GeoCo-SAVi & 0.1 & 0.08325 & 0.02485 & 0.29851 & 0.00625 & 51600 \\
GeoCo-SAVi & 0.2 & 0.15475 & 0.05029 & 0.32496 & 0.01908 & 51600 \\
GeoCo-SAVi & 0.3 & 0.22067 & 0.08241 & 0.37347 & 0.04146 & 51600 \\
\bottomrule
\end{tabular}
\end{table}
\begin{table}[!htb]
\centering
\small
\papertablecaption{MOVi-C no-op F1 (\%): unchanged recipient masks scored against the requested edit targets, on the same valid edits without boundary filtering.}
\label{tab:movic-new-noop}
\setlength{\tabcolsep}{4pt}
\renewcommand{\arraystretch}{1.18}
\begin{tabular}{lrrr}
\toprule
\rowcolor{tableheadgray}
Method & $F_{\rm pos}^{\rm noop}$ & $F_{\rm scl}^{\rm noop}$ & $F_{\rm app}^{\rm noop}$ \\
\midrule
STAITUS & 48.06 & 64.07 & 59.21 \\
GeoCo-SAVi & 29.48 & 65.69 & 60.18 \\
\bottomrule
\end{tabular}
\end{table}
\subsection{Whole-Video Editing}
\label{app:video-tube-f1}
A separate experiment follows one deterministically selected GT track per
video. Hungarian association at its maximum-visible-area frame fixes a
slot index, which is then edited throughout all 24 recurrent states.
The selected object must be valid at the reference frame, with the
support-area criterion scaled to the $128\times128$ measurement resolution.
Both methods have 247 eligible videos; 238 have the same eligible recipient
and donor video pair across methods. Table~\ref{tab:video-tube-f1} reports
this common set. The donor is the next eligible video in cyclic order,
and its appearance stream $\mathbf a_t$ is transplanted into the
recipient's geometry stream $(\mathbf p_t,s_t)$.

The four translation commands and four nonidentity scale commands are
held constant through the video. Targets are ideal transformations of
each frame's factual support; transplantation transforms the donor's
support using the corresponding frame's geometry. GT anchors the slot association at the reference frame. All 24 frames contribute, including empty edited masks and
occluded frames. The tube score follows Eq.~\ref{eq:shared-tube-f1} with
$T=24$. Valid-edit target area is measured over the full tube using the
same resolution-scaled criterion. Scores are averaged equally over videos and then over commands within
each operation. The no-op scores the unchanged recipient sequence against
the requested edit target on the same population.

\begin{table}[htbp]
\centering
\small
\papertablecaption{Whole-video tube F1 (\%) on 238 common recipient--donor video pairs, with all 24 frames retained. Brackets give 95\% confidence intervals for the main-table F-scores using the resampling procedure in Sec.~\ref{app:fscore-intervals}.}
\label{tab:video-tube-f1}
\setlength{\tabcolsep}{4pt}
\renewcommand{\arraystretch}{1.18}
\begin{tabular}{lrrr}
\toprule
\rowcolor{tableheadgray}
Method & $F_{\rm pos}$ & $F_{\rm scl}$ & $F_{\rm app}$ \\
\midrule
STAITUS & 75.34 & 72.14 & 45.15 \\
& [73.46, 77.12] & [70.14, 74.01] & [39.71, 50.41] \\
GeoCo-SAVi & 86.51 & 84.93 & 74.01 \\
& [85.68, 87.31] & [84.03, 85.82] & [70.48, 77.26] \\
STAITUS no-op & 54.60 & 61.77 & 47.20 \\
GeoCo-SAVi no-op & 40.99 & 64.60 & 58.39 \\
\bottomrule
\end{tabular}
\end{table}

\subsection{Object-Matched Editing and RGB Preservation}
\label{app:movic-edit-validation}
We compare edits of common GT objects using the reference-frame
associations in Sec.~\ref{app:movic}. The four translations and four
nonidentity scales follow Sec.~\ref{app:edit-mask-f1}.
Each measure is averaged within source videos and then equally over videos.
Table~\ref{tab:movic-edit-validation} reports the three evaluations below.

\paragraph{Common-object editing.}
The common pool requires valid reference edits and valid edits for both
associated decoded supports, as in Sec.~\ref{app:obj3d-edit-validation}.
The support-area criterion is scaled to the $128\times128$ measurement
resolution. Without boundary filtering, the pool contains 5,801 translation
and 5,591 scale edits, each spanning 1,496 GT objects and all 250 videos.
Editing F1 follows Eq.~\ref{eq:shared-editing-f1}.

\paragraph{Independent RGB targets.}
We construct visible-layer RGB oracle targets using official GT masks
and the compositing procedure in Sec.~\ref{app:obj3d-edit-validation}.
GT masks are area-resized to $64\times64$ mattes, and scaling is centered
at the GT centroid. The 1,541 valid source objects across 250 videos give
6,164 edits per command family. Both methods use identical targets.
PSNR is measured over the full image and the source/target support union
expanded by two four-connected pixel steps. Transplantation uses the
reference-frame associations above and the donor-pairing and visible-layer
procedure in Sec.~\ref{app:obj3d-edit-validation}, retaining all 1,541 valid
reference pairs without boundary filtering.

\paragraph{Non-target RGB preservation.}
We evaluate 178 common factual associations from 125 videos, with
reference-core recall and purity at least 0.5 for both methods.
All eight position and scale commands are retained, giving 712 edits per family.
For these commands, the measurement region excludes the factual/ideal GT support union,
nearest-resized to $64\times64$ and dilated with a $5\times5$ square.
Applying the same association criterion to both members of each transplant
pair retains 14 pairs from 14 recipient videos; changes are measured
outside the transplant editing region defined above.
We report RGB MAE and the percentage of pixels whose channel-averaged
absolute change exceeds $2/255$, as in
Sec.~\ref{app:obj3d-edit-validation}. Paired 95\% confidence intervals
from 5,000 source-video bootstrap replicates, with fixed donor assignments
for transplantation, exclude zero for all
reported improvements in Table~\ref{tab:movic-edit-validation}.

\begin{table}[!htb]
\centering
\small
\papertablecaption{MOVi-C object-matched editing and RGB preservation. F1 and changed-pixel rates are percentages; oracle PSNR is in dB. Each block uses common object--command pairs across methods.}
\label{tab:movic-edit-validation}
\setlength{\tabcolsep}{8pt}
\renewcommand{\arraystretch}{1.05}
\setlength{\tabcolsep}{3pt}
\tablesubheading{Editing evaluation}
\begin{tabular*}{\linewidth}{@{\extracolsep{\fill}}lrrrrrrrr@{}}
\toprule
& \multicolumn{2}{c}{F1 (\%)$\uparrow$} & \multicolumn{3}{c}{Oracle PSNR: full image$\uparrow$} & \multicolumn{3}{c}{Oracle PSNR: editing region$\uparrow$} \\
\cmidrule(lr){2-3}\cmidrule(lr){4-6}\cmidrule(lr){7-9}
Method & $F_{\rm pos}$ & $F_{\rm scl}$ & Position & Scale & Transplant & Position & Scale & Transplant \\
\midrule
STAITUS & 77.03 & 73.50 & 20.444 & 20.587 & 19.000 & 17.202 & 17.796 & 15.234 \\
GeoCo-SAVi & \textbf{85.37} & \textbf{84.65} & \textbf{21.929} & \textbf{21.983} & \textbf{19.932} & \textbf{18.607} & \textbf{19.145} & \textbf{15.887} \\
\bottomrule
\end{tabular*}
\par\medskip
\setlength{\tabcolsep}{6pt}
\tablesubheading{Non-target RGB preservation}
\begin{tabular*}{\linewidth}{@{\extracolsep{\fill}}lrrrrrr@{}}
\toprule
& \multicolumn{3}{c}{RGB MAE $\times10^3\downarrow$} & \multicolumn{3}{c}{Changed pixels (\%)$\downarrow$} \\
\cmidrule(lr){2-4}\cmidrule(lr){5-7}
Method & Position & Scale & Transplant & Position & Scale & Transplant \\
\midrule
STAITUS & 1.9314 & 2.5208 & 11.1853 & 4.3568 & 6.0451 & 20.4504 \\
GeoCo-SAVi & \textbf{0.5603} & \textbf{0.7574} & \textbf{0.8834} & \textbf{1.2508} & \textbf{1.8191} & \textbf{2.2499} \\
\bottomrule
\end{tabular*}
\end{table}

\endgroup

\FloatBarrier
\section{Factual Scale Calibration}
\label{app:factual-scale-calibration}

We assess the relationship between inferred scale $s$ and factual decoded
hard-support RMS radius $r$ over valid non-background objects on the native
canvas. Each frame--slot observation is counted once, without boundary
filtering or donor--recipient weighting. We fit $r=as+b$ by least squares
and $\log r=\log s+b$ with slope fixed to one. For each fit, we report
$R^2=1-\sum_i(y_i-\widehat y_i)^2/\sum_i(y_i-\overline y)^2$ in its
respective response space ($r$ or $\log r$).
The standard deviation of $\log(r/s)$ summarizes variation in the
radius-to-scale ratio. These statistics describe factual scale calibration;
editing responses are evaluated separately in the preceding sections.

\begin{table}[htbp]
\centering
\small
\papertablecaption{Factual scale calibration on Obj3D and MOVi-C. $N$ counts valid factual frame--slot observations. The two $R^2$ columns use their respective response spaces.}
\label{tab:factual-scale-calibration}
\setlength{\tabcolsep}{5pt}
\renewcommand{\arraystretch}{1.18}
\begin{tabular}{llrrrr}
\toprule
\rowcolor{tableheadgray}
Dataset & Method & $N$ & Linear $R^2\uparrow$ & Unit-slope log $R^2\uparrow$ & $\mathrm{Std}[\log(r/s)]\downarrow$ \\
\midrule
Obj3D & ISA & 9424 & 0.6637 & 0.5078 & 0.2093 \\
& SlotAug & 9357 & \textbf{0.8129} & 0.6399 & 0.1551 \\
& GeoCo-SAVi & 9320 & 0.7924 & \textbf{0.6873} & \textbf{0.1394} \\
\midrule
MOVi-C & STAITUS & 51971 & 0.6444 & 0.6346 & 0.3164 \\
& GeoCo-SAVi & 51629 & \textbf{0.8589} & \textbf{0.8346} & \textbf{0.2041} \\
\bottomrule
\end{tabular}
\end{table}

\section{Finite-Grid Covariance Diagnostics}
\label{app:finite-grid-diagnostics}
We evaluate the per-slot raw-logit covariance in
Proposition~\ref{prop:raw-logit-equivariance} using native-grid interior
measurements and an extended-canvas diagnostic. Each object receives four
translations and four nonidentity scales. With edited logits $\ell_e$,
bilinearly inverse-warped reference logits $\ell_t$, and measurement domain
$D$, the relative error is
\[
 \epsilon_\ell^{\rm rel}=
 \frac{\sqrt{|D|^{-1}\sum_{\mathbf u\in D}
              (\ell_e(\mathbf u)-\ell_t(\mathbf u))^2}}
 {\max\{\sqrt{|D|^{-1}\sum_{\mathbf u\in D}\ell_t(\mathbf u)^2},10^{-8}\}}.
\]
Statistics are averaged over interventions with at least 16 measurement
pixels, using FP32 with TF32 disabled. Both diagnostics use the same
frozen slots and model weights.

This diagnostic uses a separate factual selection. Obj3D selects one
highest-confidence qualified slot per window, giving 2,581 objects from
2,592 evaluated windows and 10,324 interventions per edit family.
Qualification requires hard area between 20 pixels and 60\% of the canvas,
alpha peak above 0.3, scale below 0.3, and an available background slot.
Normalized attention has effective area between 2 and 75\% of the tokens,
maximum cosine with another slot below 0.5, and a 2.5-standard-deviation
extent around the explicit position within a two-pixel canvas margin.
Among qualified slots, confidence is one minus the maximum attention cosine.
MOVi-C uses 178 common GT-associated objects with reference-core recall
and purity at least 0.5, an available background, and a two-pixel
reference bounding-box margin. These populations characterize numerical
covariance on localized factual objects; all-slot editing results above
use their own stated populations.

\paragraph{Native output, receptive-field-valid domain.}
We propagate validity through each bilinear upsampling and scale-steered
convolution, including the alpha head. The measurement domain contains
pixels for which the edited field and all bilinear contributors to the
warped reference have complete in-canvas computational support
(validity at least $1-10^{-6}$). Native rendered fields are unchanged.
Domain membership depends on geometry and operator support.

\paragraph{Extended computational canvas, original measurement domain.}
We extend the per-layer coordinate domain from $[-1,1]^2$ to $[-3,3]^2$,
preserving lattice spacing, weights, coordinate functions, and kernel
offsets. A layer with native side length $H$ thus uses $3(H-1)+1$ samples.
Factual and edited fields are computed on this extended canvas and
measured at original-grid pixels whose inverse-transformed coordinates
lie in the original factual canvas. This computational diagnostic uses
no receptive-field restriction on the measurement pixels.

\begin{table}[!htb]
\centering
\small
\papertablecaption{Finite-grid covariance and final-output diagnostics.
Raw-logit columns report relative RMSE in percent, with the number of
interventions having at least 16 measurement pixels in parentheses.
Native uses the receptive-field-valid domain; extended uses the extended
computational canvas. Final-output columns use the extended canvas:
probability RMSE averages all 10,324 interventions per family on Obj3D
and 712 on MOVi-C; F1 and its lower bound are means in percent over the
$N_{\mathrm{F1}}$ target-membership comparisons with $|A|+|I|>0$.}
\label{tab:finite-grid-errors}
\label{tab:final-output-diagnostics}
\setlength{\tabcolsep}{3pt}
\renewcommand{\arraystretch}{1.18}
\begin{tabular}{llrrrrrr}
\toprule
& & \multicolumn{2}{c}{Raw-logit RMSE (\%), $N$} & \multicolumn{4}{c}{Extended-canvas final outputs} \\
\cmidrule(lr){3-4}\cmidrule(lr){5-8}
\rowcolor{tableheadgray}
Dataset & Edit & Native & Extended & Prob.\ RMSE & F1 & F1 bound & $N_{\mathrm{F1}}$ \\
\midrule
Obj3D & Translation & 0.0325 (10324) & 0.0199 (10324) & 0.00422 & 95.87 & 95.56 & 10324 \\
& Scaling & 2.0376 (10324) & 0.9230 (10324) & 0.01049 & 95.69 & 95.17 & 10324 \\
\midrule
MOVi-C & Translation & 0.0515 (704) & 0.0446 (712) & 0.00042 & 99.86 & 99.79 & 711 \\
& Scaling & 2.4478 (689) & 1.4237 (712) & 0.01136 & 96.95 & 95.16 & 712 \\
\bottomrule
\end{tabular}
\end{table}

Table~\ref{tab:finite-grid-errors} reports both raw-logit diagnostics on Obj3D and
MOVi-C. Translation errors are below 0.06\% in both settings; scaling
errors range from 2.04--2.45\% on the native receptive-field-valid domain
and 0.92--1.42\% with the extended computational canvas. These pixelwise
measurement domains are separate from the object-level boundary scopes
used for editing metrics.

\subsection{Final Probabilities and Target Membership}
\label{app:final-output-bounds}
Let $\ell_j^f$ and $\ell_j^e$ denote factual and edited final mask logits,
and let $P$ bilinearly sample a factual field at the inverse edit coordinate.
For target slot $i$, define $q_i=P\ell_i^f$ and
$\varepsilon_i=\ell_i^e-q_i$. The reference scene has target logit $q_i$
and edited competing logits $\ell_j^e$, $j\ne i$.
Writing $b_e=\logsumexp_{j\ne i}\ell_j^e$ and
$m_e=\max_{j\ne i}\ell_j^e$, its target probability is
$\bar\alpha_i=\sigma(q_i-b_e)$, where $\sigma$ is the logistic sigmoid.

\newpage
\begin{proposition}[Finite-grid probability and membership bounds]
\label{prop:final-output-bounds}
At each measurement pixel, the final full-softmax probability satisfies
\begin{equation}
 |\alpha_i^e-\bar\alpha_i|
 \le \tanh\!\left(\frac{|\varepsilon_i|}{4}\right)
 \le \frac{|\varepsilon_i|}{4}.
 \label{eq:final-probability-bound}
\end{equation}
The actual and reference scenes assign the same target membership whenever
\begin{equation}
 |q_i-m_e|>|\varepsilon_i|.
 \label{eq:hard-membership-condition}
\end{equation}
Let $A$ and $I$ be their respective argmax target masks on domain $D$, and
let $U=\{\mathbf u\in D:|q_i-m_e|\le|\varepsilon_i|\}$.
For $|A|+|I|>0$, their overlap satisfies
\begin{equation}
 \mathrm{F1}(A,I)\ge
 \max\!\left\{0,1-\frac{|U\cap(A\cup I)|}{|A|+|I|}\right\}.
 \label{eq:membership-f1-bound}
\end{equation}
\end{proposition}

\begin{proof}
Full softmax gives $\alpha_i^e=\sigma(q_i-b_e+\varepsilon_i)$.
For any displacement $h$, the maximum of
$|\sigma(x+h)-\sigma(x)|$ over $x$ is $\tanh(|h|/4)$;
the derivative bound $\sigma'(x)\le1/4$ gives the second inequality.
Condition~\eqref{eq:hard-membership-condition} preserves the sign of the
target logit's difference from the largest competitor. Thus
$A\mathbin{\triangle}I\subseteq U$, and
$\mathrm{F1}(A,I)=1-|A\mathbin{\triangle}I|/(|A|+|I|)$ yields
Eq.~\eqref{eq:membership-f1-bound}.
\end{proof}

The residual includes the complete final-logit path. In particular, if
$\ell_i=r_i+c_i$ separates raw logits and final corrections, then
$\varepsilon_i=(r_i^e-Pr_i^f)+(c_i^e-Pc_i^f)$.
For comparison with a transported factual probability, define
$b_f=\logsumexp_{j\ne i}\ell_j^f$, $\Delta b=b_e-Pb_f$, and
$J_i=|\sigma(q_i-Pb_f)-P\alpha_i^f|$. The same argument and the triangle
inequality give
\begin{equation}
 |\alpha_i^e-P\alpha_i^f|
 \le \tanh\!\left(\frac{|\varepsilon_i-\Delta b|}{4}\right)+J_i,
 \label{eq:transported-probability-bound}
\end{equation}
which explicitly accounts for scene competition and bilinear interpolation.

\paragraph{Measurements.}
We evaluate these relations on the same objects and commands using the
extended computational canvas and original measurement domain above.
Final probabilities include the Obj3D selector, evaluated on the original
canvas with the corresponding cropped decoder features and attention.
Table~\ref{tab:final-output-diagnostics} also reports mean probability RMSE,
target-membership F1, and the mean per-intervention lower bound from
Eq.~\eqref{eq:membership-f1-bound}. The probability inequalities hold to
numerical precision, and all pixels satisfying
Eq.~\eqref{eq:hard-membership-condition} retain target membership.

\addtocontents{toc}{\protect\setcounter{tocdepth}{3}}


\begin{thebibliography}{34}
\providecommand{\natexlab}[1]{#1}

\bibitem[{Aydemir, Xie, and G{\"u}ney(2023)}]{solv}
Aydemir, G.; Xie, W.; and G{\"u}ney, F. 2023.
\newblock Self-supervised Object-Centric Learning for Videos.
\newblock In \emph{Advances in Neural Information Processing Systems
  (NeurIPS)}.

\bibitem[{Biza et~al.(2023)Biza, van Steenkiste, Sajjadi, Elsayed, Mahendran,
  and Kipf}]{isa}
Biza, O.; van Steenkiste, S.; Sajjadi, M. S.~M.; Elsayed, G.~F.; Mahendran, A.;
  and Kipf, T. 2023.
\newblock Invariant Slot Attention: Object Discovery with Slot-Centric
  Reference Frames.
\newblock In \emph{International Conference on Machine Learning (ICML)}.

\bibitem[{Burgess et~al.(2019)Burgess, Matthey, Watters, Kabra, Higgins,
  Botvinick, and Lerchner}]{monet}
Burgess, C.~P.; Matthey, L.; Watters, N.; Kabra, R.; Higgins, I.; Botvinick,
  M.; and Lerchner, A. 2019.
\newblock {MONet}: Unsupervised Scene Decomposition and Representation.
\newblock \emph{arXiv preprint arXiv:1901.11390}.

\bibitem[{Chen et~al.(2025)Chen, Huang, Shen, Huang, Li, and Xue}]{gold}
Chen, T.; Huang, Y.; Shen, Z.; Huang, J.; Li, B.; and Xue, X. 2025.
\newblock Learning Global Object-Centric Representations via Disentangled Slot
  Attention.
\newblock \emph{Machine Learning}, 114(2): 40.

\bibitem[{Cho et~al.(2014)Cho, van Merri{\"e}nboer, Gulcehre, Bahdanau,
  Bougares, Schwenk, and Bengio}]{gru}
Cho, K.; van Merri{\"e}nboer, B.; Gulcehre, C.; Bahdanau, D.; Bougares, F.;
  Schwenk, H.; and Bengio, Y. 2014.
\newblock Learning Phrase Representations using RNN Encoder-Decoder for
  Statistical Machine Translation.
\newblock In \emph{Conference on Empirical Methods in Natural Language
  Processing (EMNLP)}.

\bibitem[{Didolkar et~al.(2025)Didolkar, Zadaianchuk, Awal, Seitzer, Gavves,
  and Agrawal}]{ctrlo}
Didolkar, A.; Zadaianchuk, A.; Awal, R.; Seitzer, M.; Gavves, E.; and Agrawal,
  A. 2025.
\newblock {CTRL-O}: Language-Controllable Object-Centric Visual Representation
  Learning.
\newblock In \emph{IEEE/CVF Conference on Computer Vision and Pattern
  Recognition (CVPR)}, 29523--29533.

\bibitem[{Engelcke et~al.(2020)Engelcke, Kosiorek, Parker~Jones, and
  Posner}]{genesis}
Engelcke, M.; Kosiorek, A.~R.; Parker~Jones, O.; and Posner, I. 2020.
\newblock GENESIS: Generative Scene Inference and Sampling with Object-Centric
  Latent Representations.
\newblock In \emph{International Conference on Learning Representations
  (ICLR)}.

\bibitem[{Eslami et~al.(2016)Eslami, Heess, Weber, Tassa, Szepesvari,
  Kavukcuoglu, and Hinton}]{air}
Eslami, S. M.~A.; Heess, N.; Weber, T.; Tassa, Y.; Szepesvari, D.; Kavukcuoglu,
  K.; and Hinton, G.~E. 2016.
\newblock Attend, Infer, Repeat: Fast Scene Understanding with Generative
  Models.
\newblock In \emph{Advances in Neural Information Processing Systems
  (NeurIPS)}.

\bibitem[{Greff et~al.(2022)Greff, Belletti, Beyer, Doersch et~al.}]{kubric}
Greff, K.; Belletti, F.; Beyer, L.; Doersch, C.; et~al. 2022.
\newblock Kubric: A scalable dataset generator.
\newblock In \emph{IEEE/CVF Conference on Computer Vision and Pattern
  Recognition (CVPR)}.

\bibitem[{Jiang et~al.(2020)Jiang, Janghorbani, de~Melo, and Ahn}]{scalor}
Jiang, J.; Janghorbani, S.; de~Melo, G.; and Ahn, S. 2020.
\newblock {SCALOR}: Generative World Models with Scalable Object
  Representations.
\newblock In \emph{International Conference on Learning Representations
  (ICLR)}.

\bibitem[{Kim et~al.(2023)Kim, Choi, Kang, Lee, Choi, and Kim}]{slotaug}
Kim, J.; Choi, J.; Kang, J.; Lee, C.; Choi, H.-J.; and Kim, S.~J. 2023.
\newblock Leveraging Image Augmentation for Object Manipulation: Towards
  Interpretable Controllability in Object-Centric Learning.
\newblock \emph{arXiv preprint arXiv:2310.08929}.

\bibitem[{Kipf et~al.(2022)Kipf, Elsayed, Mahendran, Stone, Sabour, Heigold,
  Jonschkowski, Dosovitskiy, and Greff}]{savi}
Kipf, T.; Elsayed, G.~F.; Mahendran, A.; Stone, A.; Sabour, S.; Heigold, G.;
  Jonschkowski, R.; Dosovitskiy, A.; and Greff, K. 2022.
\newblock Conditional Object-Centric Learning from Video.
\newblock In \emph{International Conference on Learning Representations
  (ICLR)}.

\bibitem[{Li et~al.(2025)Li, Ren, Liu, and Sun}]{statm}
Li, J.; Ren, P.; Liu, Y.; and Sun, H. 2025.
\newblock Reasoning-Enhanced Object-Centric Learning for Videos.
\newblock In \emph{Proceedings of the 31st ACM SIGKDD Conference on Knowledge
  Discovery and Data Mining V.1}, 659--670. Association for Computing
  Machinery.

\bibitem[{Lin et~al.(2020{\natexlab{a}})Lin, Wu, Peri, Fu, Jiang, and
  Ahn}]{gswm}
Lin, Z.; Wu, Y.-F.; Peri, S.; Fu, B.; Jiang, J.; and Ahn, S.
  2020{\natexlab{a}}.
\newblock Improving Generative Imagination in Object-Centric World Models.
\newblock In \emph{International Conference on Machine Learning (ICML)}.

\bibitem[{Lin et~al.(2020{\natexlab{b}})Lin, Wu, Peri, Sun, Singh, Deng, Jiang,
  and Ahn}]{space}
Lin, Z.; Wu, Y.-F.; Peri, S.~V.; Sun, W.; Singh, G.; Deng, F.; Jiang, J.; and
  Ahn, S. 2020{\natexlab{b}}.
\newblock {SPACE}: Unsupervised Object-Oriented Scene Representation via
  Spatial Attention and Decomposition.
\newblock In \emph{International Conference on Learning Representations
  (ICLR)}.

\bibitem[{Locatello et~al.(2020)Locatello, Weissenborn, Unterthiner, Mahendran,
  Heigold, Uszkoreit, Dosovitskiy, and Kipf}]{slotattention}
Locatello, F.; Weissenborn, D.; Unterthiner, T.; Mahendran, A.; Heigold, G.;
  Uszkoreit, J.; Dosovitskiy, A.; and Kipf, T. 2020.
\newblock Object-Centric Learning with Slot Attention.
\newblock In \emph{Advances in Neural Information Processing Systems
  (NeurIPS)}.

\bibitem[{Loshchilov and Hutter(2019)}]{adamw}
Loshchilov, I.; and Hutter, F. 2019.
\newblock Decoupled Weight Decay Regularization.
\newblock In \emph{International Conference on Learning Representations
  (ICLR)}.

\bibitem[{Majellaro et~al.(2025)Majellaro, Collu, Plaat, and Moerland}]{disa}
Majellaro, R.; Collu, J.; Plaat, A.; and Moerland, T.~M. 2025.
\newblock Explicitly Disentangled Representations in Object-Centric Learning.
\newblock \emph{Transactions on Machine Learning Research}.

\bibitem[{Manasyan et~al.(2025)Manasyan, Seitzer, Radovic, Martius, and
  Zadaianchuk}]{slotcontrast}
Manasyan, A.; Seitzer, M.; Radovic, F.; Martius, G.; and Zadaianchuk, A. 2025.
\newblock Temporally Consistent Object-Centric Learning by Contrasting Slots.
\newblock In \emph{IEEE/CVF Conference on Computer Vision and Pattern
  Recognition (CVPR)}, 5401--5411.

\bibitem[{Moon and Heo(2026)}]{ssync}
Moon, W.; and Heo, J.-P. 2026.
\newblock Selective Synergistic Learning for Video Object-Centric Learning.
\newblock \emph{arXiv preprint arXiv:2606.15527}.

\bibitem[{Moon, Seong, and Heo(2026)}]{slotcurri}
Moon, W.; Seong, H.~S.; and Heo, J.-P. 2026.
\newblock Reconstruction-Guided Slot Curriculum: Addressing Object
  Over-Fragmentation in Video Object-Centric Learning.
\newblock In \emph{IEEE/CVF Conference on Computer Vision and Pattern
  Recognition (CVPR)}, 25001--25010.

\bibitem[{Oquab et~al.(2024)Oquab, Darcet, Moutakanni, Vo, Szafraniec,
  Khalidov, Fernandez, Haziza, Massa, El-Nouby, Assran, Ballas, Galuba, Howes,
  Huang, Li, Misra, Rabbat, Sharma, Synnaeve, Xu, Jegou, Mairal, Labatut,
  Joulin, and Bojanowski}]{dinov2}
Oquab, M.; Darcet, T.; Moutakanni, T.; Vo, H.; Szafraniec, M.; Khalidov, V.;
  Fernandez, P.; Haziza, D.; Massa, F.; El-Nouby, A.; Assran, M.; Ballas, N.;
  Galuba, W.; Howes, R.; Huang, P.-Y.; Li, S.-W.; Misra, I.; Rabbat, M.;
  Sharma, V.; Synnaeve, G.; Xu, H.; Jegou, H.; Mairal, J.; Labatut, P.; Joulin,
  A.; and Bojanowski, P. 2024.
\newblock {DINOv2}: Learning Robust Visual Features without Supervision.
\newblock \emph{Transactions on Machine Learning Research}.

\bibitem[{Seitzer et~al.(2023)Seitzer, Horn, Zadaianchuk, Zietlow, Xiao,
  Simon-Gabriel, He, Zhang, Sch{\"o}lkopf, Brox, and Locatello}]{dinosaur}
Seitzer, M.; Horn, M.; Zadaianchuk, A.; Zietlow, D.; Xiao, T.; Simon-Gabriel,
  C.-J.; He, T.; Zhang, Z.; Sch{\"o}lkopf, B.; Brox, T.; and Locatello, F.
  2023.
\newblock Bridging the Gap to Real-World Object-Centric Learning.
\newblock In \emph{International Conference on Learning Representations
  (ICLR)}.

\bibitem[{Seong, Moon, and Heo(2026)}]{srl}
Seong, H.~S.; Moon, W.; and Heo, J.-P. 2026.
\newblock From Vicious to Virtuous Cycles: Synergistic Representation Learning
  for Unsupervised Video Object-Centric Learning.
\newblock In \emph{International Conference on Learning Representations
  (ICLR)}.

\bibitem[{Sosnovik, Moskalev, and Smeulders(2021)}]{disco}
Sosnovik, I.; Moskalev, A.; and Smeulders, A. W.~M. 2021.
\newblock {DISCO}: Accurate Discrete Scale Convolutions.
\newblock In \emph{British Machine Vision Conference (BMVC)}.

\bibitem[{Sosnovik, Szmaja, and Smeulders(2020)}]{sesn}
Sosnovik, I.; Szmaja, M.; and Smeulders, A. W.~M. 2020.
\newblock Scale-Equivariant Steerable Networks.
\newblock In \emph{International Conference on Learning Representations
  (ICLR)}.

\bibitem[{Tian et~al.(2025)Tian, Yang, Yu, and Kot}]{foregroundocl}
Tian, P.; Yang, S.; Yu, H.; and Kot, A.~C. 2025.
\newblock Pay Attention to the Foreground in Object-Centric Learning.
\newblock In \emph{IEEE/CVF Conference on Computer Vision and Pattern
  Recognition (CVPR)}, 30281--30290.

\bibitem[{Tran et~al.(2026)Tran, Nguyen, Vo, Vo, and Le}]{dssa}
Tran, S.; Nguyen, D.; Vo, H.; Vo, K.; and Le, N. 2026.
\newblock Dual-State Slot Attention: Decoupling Appearance and Identity for
  Video Object-Centric Learning.
\newblock \emph{arXiv preprint arXiv:2606.12601}.

\bibitem[{Vaswani et~al.(2017)Vaswani, Shazeer, Parmar, Uszkoreit, Jones,
  Gomez, Kaiser, and Polosukhin}]{attention}
Vaswani, A.; Shazeer, N.; Parmar, N.; Uszkoreit, J.; Jones, L.; Gomez, A.~N.;
  Kaiser, L.; and Polosukhin, I. 2017.
\newblock Attention Is All You Need.
\newblock \emph{arXiv preprint arXiv:1706.03762}.

\bibitem[{Watters et~al.(2019)Watters, Matthey, Burgess, and
  Lerchner}]{spatialbroadcast}
Watters, N.; Matthey, L.; Burgess, C.~P.; and Lerchner, A. 2019.
\newblock Spatial Broadcast Decoder: A Simple Architecture for Learning
  Disentangled Representations in VAEs.
\newblock \emph{arXiv preprint arXiv:1901.07017}.

\bibitem[{Wei, Nejjar, and Fink(2026)}]{staitus}
Wei, A.; Nejjar, I.; and Fink, O. 2026.
\newblock Rethinking Object-Centric Representations for Video Dynamics
  Modeling.
\newblock \emph{arXiv preprint arXiv:2606.23436}.

\bibitem[{Worrall and Welling(2019)}]{deepscalespaces}
Worrall, D.~E.; and Welling, M. 2019.
\newblock Deep Scale-spaces: Equivariance Over Scale.
\newblock In \emph{Advances in Neural Information Processing Systems
  (NeurIPS)}, volume~32, 7364--7376.

\bibitem[{Zadaianchuk, Seitzer, and Martius(2023)}]{videosaur}
Zadaianchuk, A.; Seitzer, M.; and Martius, G. 2023.
\newblock Object-Centric Learning for Real-World Videos by Predicting Temporal
  Feature Similarities.
\newblock In \emph{Advances in Neural Information Processing Systems
  (NeurIPS)}.

\bibitem[{Zhang et~al.(2017)Zhang, Tang, Zhang, Li, and Yan}]{sac}
Zhang, R.; Tang, S.; Zhang, Y.; Li, J.; and Yan, S. 2017.
\newblock Scale-Adaptive Convolutions for Scene Parsing.
\newblock In \emph{IEEE International Conference on Computer Vision (ICCV)}.

\end{thebibliography}
\end{document}